\documentclass{article}
\usepackage[usenames,dvipsnames]{xcolor}
\usepackage[final]{corl_2021} 

\usepackage[ruled,vlined,noline, noend]{algorithm2e}
\usepackage{amsmath}
\usepackage{amssymb}
\usepackage{booktabs}
\usepackage{multirow}
\usepackage{graphicx}
\usepackage{subcaption}
\usepackage{paralist}
\usepackage[T1]{fontenc}
\usepackage[font=scriptsize,labelfont=sl,labelsep=period,tableposition=top]{caption}
\usepackage{wrapfig}
\usepackage{enumitem}
\usepackage[subtle]{savetrees}
\usepackage{amsthm}
\usepackage{float}
\usepackage{svg}
\usepackage[ruled,vlined]{algorithm2e}

\usepackage{blindtext}

\newcommand{\ie}{i.e., }
\newcommand{\eg}{e.g., }

\DeclareCaptionLabelFormat{andtable}{#1~#2  \&  \tablename~\thetable}

\newcommand{\Real}{\mathbb{R}}
\newcommand{\Rone}{\Real^1}
\newcommand{\Rn}{\Real^n}
\newcommand{\Rm}{\Real^m}
\newcommand{\Rmn}{\Real^{m+n}}
\newcommand{\bfx}{\mathbf{x}}
\newcommand{\bfy}{\mathbf{y}}

\newcommand{\bfa}{\mathbf{a}}
\newcommand{\bfo}{\mathbf{o}}
\newcommand{\bfxy}{\bfx,\bfy}
\newcommand{\bfxFx}{\bfx,F(\bfx)}
\newcommand{\bfu}{\mathbf{u}}
\newcommand{\bfv}{\mathbf{v}}
\newcommand{\argmin}{\text{argmin}}
\newcommand{\argminy}{\argmin_\bfy}
\newcommand{\PRn}{P(\Rn) \setminus \{ \emptyset \}}

\newcommand{\pete}[1]{\ifdefined\DRAFT \textcolor{cyan}{Pete: #1} \else \fi}

\newcommand{\johnny}[1]{\ifdefined\DRAFT \textcolor{purple}{Johnny: #1} \else \fi}
\newcommand{\igor}[1]{\ifdefined\DRAFT \textcolor{brown}{Igor: #1} \else \fi}
\definecolor{lauragreen}{RGB}{0, 120, 0}

\usepackage{pifont}
\newcommand{\cmark}{\textcolor[HTML]{59a14f}{\ding{51}}}%
\newcommand{\xmark}{\textcolor[HTML]{e15759}{\ding{55}}}%
\newtheorem{Theorem}{Theorem}

\newtheorem{lemma}[Theorem]{Lemma}
\makeatletter

\newenvironment{customTheorem}[1]
  {\count@\c@Theorem
   \global\c@Theorem#1 %
    \global\advance\c@Theorem\m@ne
   \Theorem}
  {\endTheorem
   \global\c@Theorem\count@}

\makeatother

\title{Implicit Behavioral Cloning}

\author{
  \textbf{Pete Florence, Corey Lynch, Andy Zeng, Oscar Ramirez, Ayzaan Wahid,}\\
  \textbf{Laura Downs, Adrian Wong, Johnny Lee, Igor Mordatch, Jonathan Tompson}\\ \\
  Robotics at Google
}

\begin{document}
\maketitle


\vspace{-2.5em}
\begin{abstract}
We find that across a wide range of robot policy learning scenarios, treating supervised policy learning with an {\em{implicit model}} generally performs better, on average, than
commonly used explicit models.  We present extensive experiments on this finding, and we provide both intuitive insight and theoretical arguments distinguishing the properties of implicit models compared to their explicit counterparts, particularly with respect to approximating complex, potentially discontinuous and multi-valued (set-valued) functions.
On robotic policy learning tasks we show that implicit behavioral cloning policies with energy-based models (EBM) often outperform common explicit (Mean Square Error, or Mixture Density) behavioral cloning policies, including on tasks with high-dimensional action spaces and visual image inputs. We find these policies provide competitive results or outperform state-of-the-art offline reinforcement learning methods on the challenging human-expert tasks from the D4RL benchmark suite, despite using no reward information. In the real world, robots with implicit policies can learn complex and remarkably subtle behaviors on contact-rich tasks from human demonstrations, including tasks with high combinatorial complexity and tasks requiring 1mm precision. 





\end{abstract}



\addtocontents{toc}{\protect\setcounter{tocdepth}{-1}} 
\section{Introduction}

Behavioral cloning (BC) \cite{pomerleau1989alvinn} remains one of the simplest machine learning methods to acquire robotic skills in the real world.
BC casts the imitation of expert demonstrations as a supervised learning problem, and despite valid concerns (both empirical and theoretical) about its shortcomings (\eg compounding errors \cite{ross2011reduction, tu2021closing}), in practice it enables some of the most compelling results of real robots generalizing complex behaviors to new unstructured scenarios \cite{zhang2018deep, florence2019self, zeng2020tossingbot}.
%
Although considerable research has been devoted to developing new imitation learning methods \cite{ho2016generative,abbeel2004apprenticeship,ho2016model} to address BC's known limitations, here we investigate a fundamental design decision that has largely been overlooked: the form of the policy itself.
Like many other supervised learning methods, BC policies are often represented by explicit continuous feed-forward models (\eg deep networks) of the form $\hat{\bfa} = F_{\theta}(\bfo)$ that map directly from input observations $\bfo$ to output actions $\bfa \in {\mathcal{A}}$. But what if $F_{\theta}$ is the wrong choice? 
%
%

In this work, we propose to reformulate BC using {\em{implicit models}} -- specifically, the composition of $\arg\min$ with a continuous energy function $E_{\theta}$ (see Sec.~\ref{sec:background} for definition) to represent the policy $\pi_{\theta}$:
$$\hat{\bfa} = \underset{\bfa \in \mathcal{A}}{\arg\min} \ \ E_{\theta}(\bfo,\bfa) \qquad \text{instead of} \qquad \hat{\bfa} = F_{\theta}(\bfo) \ \ .$$
This formulates imitation as a conditional energy-based modeling (EBM) problem \cite{lecun2006tutorial} (Fig.~\ref{fig:teaser}), and at inference time (given $\bfo$) performs implicit regression by optimizing for the optimal action $\hat{\bfa}$ via sampling or gradient descent \cite{welling2011bayesian,du2019implicit}.
While implicit models have been used as partial components (\eg value functions) for various reinforcement learning (RL) methods \cite{haarnoja2017reinforcement, du2020planning, kostrikov2021offline, nachum2021provable}, our work presents a distinct yet simple method: 
do BC with implicit models. Further, this enables a unique case study that investigates the choice between implicit vs. explicit policies that may inform other policy learning settings beyond BC.

Our experiments show that this simple change can lead to remarkable improvements in performance across a wide range of contact-rich tasks: from bi-manually scooping piles of small objects into bowls with spatulas, to precisely pushing blocks into fixtures with tight 1mm tolerances, to sorting mixed collections of blocks by their colors. Results show that implicit models for BC exhibit the capacity to learn long-horizon, closed-loop visuomotor tasks better than their explicit counterparts -- and surprisingly, give rise to a new class of BC baselines that are competitive with state-of-the-art offline RL algorithms on standard simulated benchmarks \cite{fu2020d4rl}.
To shed light on these results, we provide observations on the intuitive properties of implicit models, and present theoretical justification that we believe are highly relevant to part of their success: their ability to represent not only multi-modal distributions, but also discontinuous functions.
%
%

\begin{figure}[t]
  \centering
  \vspace{-1em}
  \begin{minipage}[b]{1.0\textwidth}
    \includegraphics[width=\textwidth]{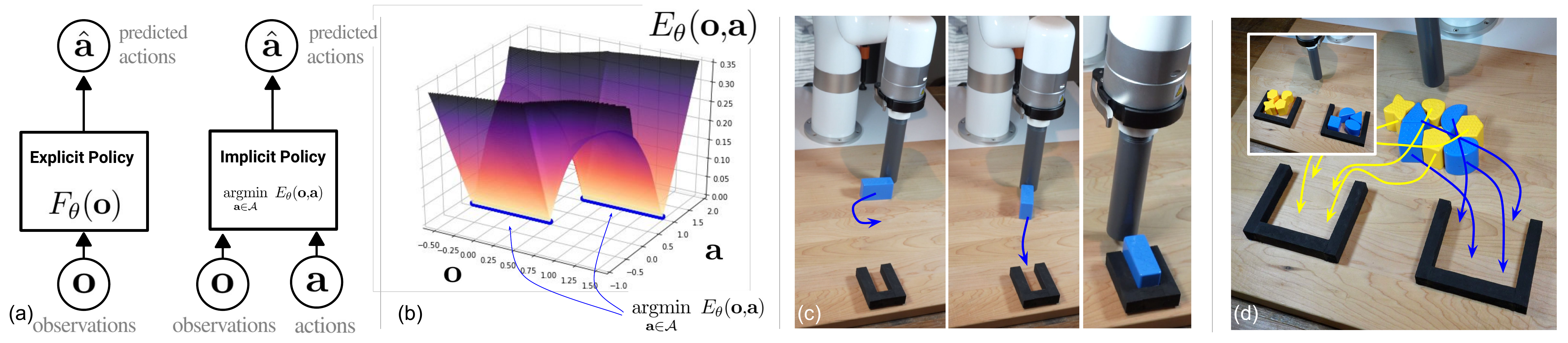}
    \caption{(a) In contrast to explicit policies, implicit policies leverage parameterized energy functions that take both observations (e.g. images) and actions as inputs, and optimize for actions that minimize the energy landscape (b). For learning complex, closed-loop, multimodal visuomotor tasks such as precise block insertion (c) and sorting (d) from human demonstrations, implicit policies perform substantially better than explicit ones.}
  \end{minipage}
  \vspace{-2.0em}
\end{figure} \label{fig:teaser}

\textbf{Paper Organization.}  After a brief background (Sec.~\ref{sec:background}), to build intuition on 
the nature of implicit models,
we present their empirical properties (Sec.~\ref{sec:empirical-properties}).  We then present our main results with policy learning (Sec.~\ref{sec:policy-results}), both in simulated tasks and in the real world. Inspired by these results, we provide theoretical insight (Sec.~\ref{sec:theory}), followed by related work (Sec.~\ref{sec:related}) and conclusions (Sec.~\ref{sec:conclusion}).

\section{Background: Implicit Model Training and Inference}\label{sec:background}

We define an {\em{implicit model}} as any composition $(\arg\min_\bfy \ \circ \ E_{\theta}(\bfx,\bfy))$, in which inference is performed using some general-purpose function approximator $E: \mathbb{R}^{m + n} \rightarrow \mathbb{R}^1$ to solve the optimization problem $\hat{\bfy} = \arg\min_{\bfy} \ E_{\theta}(\bfx, \bfy)$. 
We use techniques from the energy-based model (EBM) literature to train such a model.
Given a dataset of samples $\{\bfx_i, \bfy_i\}$, and regression bounds $\mathbf{y}_{\text{min}}, \mathbf{y}_{\text{max}} \in \mathbb{R}^m$, training consists of generating a set of negative counter-examples ${\color{red}\{\tilde{\mathbf{y}}^j_i\}_{j=1}^{N_{\text{neg.}}}  }$ for each sample $\bfx_i$ in a batch, and employing an InfoNCE-style \cite{oord2018representation} loss function. This loss equates to the negative log likelihood of $p_{\theta}(\mathbf{y} |\mathbf{x}) = \frac{\exp(-E_{\theta}(\mathbf{x}, \mathbf{y}))}{Z(\mathbf{x}, \theta)}$, and the counter-examples are used to estimate $Z(\mathbf{x}_i, \theta)$: 
$$\mathcal{L}_{\text{InfoNCE}} = \sum_{i=1}^N -\log \big( \tilde{p}_{\theta}( {\color{black} \mathbf{y}_i} | \ \mathbf{x}, \ {\color{red}\{\tilde{\mathbf{y}}^j_i\}_{j=1}^{N_{\text{neg.}}}  } ) \big), \quad   
\tilde{p}_{\theta}( {\color{black} \mathbf{y}_i} | \ \mathbf{x}, \ {\color{red}\{\tilde{\mathbf{y}}^j_i\}_{j=1}^{N_{\text{neg.}}}  } ) =  \frac{e^{-E_{\theta}(\mathbf{x}_i, {\color{black} \mathbf{y}_i}  )}} {e^{-E_{\theta}( \mathbf{x}_i, {\color{black} \mathbf{y}_i})} +  {\color{red} \sum_{j=1}^{N_{\text{neg}}}} e^{-E_{\theta}(\mathbf{x}_i, {\color{red} \tilde{\mathbf{y}}^j_i} )} }
$$
With a trained energy model $E_{\theta}(\bfx,\bfy)$, implicit inference can be performed with stochastic optimization to solve $\hat{\bfy} = \arg\min_{\bfy} \ E_{\theta}(\bfx, \bfy)$. To demonstrate a breadth of approaches, we present results with three different EBM training and inference methods discussed below, however a comprehensive comparison of all EBM variants is outside the scope of this paper; see \cite{song2021train} for a comprehensive reference. We use either a) a derivative-free (sampling-based) optimization procedure, b) an auto-regressive variant of the derivative-free optimizer which performs coordinate descent, or c) gradient-based Langevin sampling~\cite{welling2011bayesian,du2019implicit} with gradient penalty~\cite{gradientpenalty2021} loss during training -- see the Appendix  for descriptions and comparisons of these choices.

\section{Intriguing Properties of Implicit vs. Explicit Models}\label{sec:empirical-properties}


Consider an explicit model $\bfy=f_{\theta}(\bfx)$, and an implicit model $\arg\min_\bfy E_{\theta}(\bfx,\bfy)$ where both $f_{\theta}(\cdot)$ and $E_{\theta}(\cdot)$ are represented by almost-identical network architectures.
Comparing these models, we examine: (i) how do they perform near discontinuities?, (ii) how do they fit multi-valued functions?, and (iii) how do they extrapolate?  For both $f_{\theta}$ and $E_{\theta}$ we use almost-identical ReLU-activation fully-connected Multi-Layer Perceptrons (MLPs), with the only difference being the additional input of $\bfy$ in the latter.
Explicit ``MSE'' models are trained with Mean Square Error (MSE), explicit ``MDN'' models are Mixture Density Networks (MDN) \cite{bishop1994mixture}, and implicit ``EBM'' models are trained with $\mathcal{L}_{\text{InfoNCE}}$ and optimized with derivative-free optimization.
Figs.~\ref{fig:curve_fitting},~\ref{fig:curve_fitting_multivalued} show models trained on a number of $\mathbb{R}^1 \rightarrow \mathbb{R}^1$ functions (Fig.~\ref{fig:curve_fitting}) and multi-valued functions (Fig.~\ref{fig:curve_fitting_multivalued}). For each of these we examine regions of discontinuities, multi-modalities, and/or extrapolation. 
%
%
\begin{figure}
\vspace{-0.5em}
\centering
\hspace*{-0.50cm} 
  \includegraphics[width=1.1\textwidth]{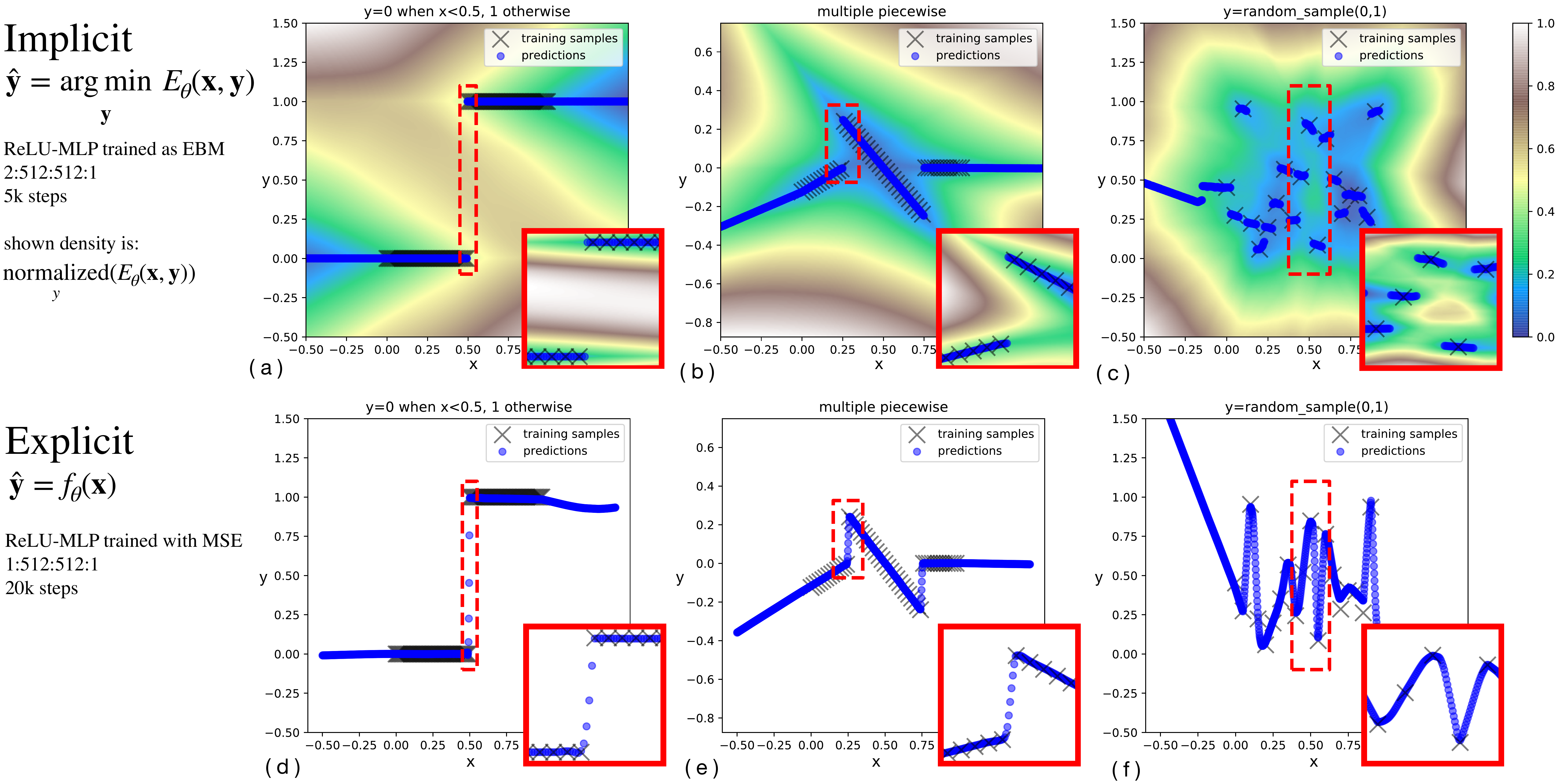}
  \vspace{-0.5cm}
  \caption{Comparison between implicit vs explicit learning of 1D functions, $\mathbb{R}^1 \rightarrow \mathbb{R}^1$, showing extrapolation (outside of $x=[0,1]$) behavior beyond training samples and detailed views (red insets) of interpolation behavior at discontinuities. (a,d) Single discontinuity between constant values; (b,e) piecewise continuous sections with differing $\frac{dy}{dx}$, (c,f) random Gaussian noise, for unregularized models.}
  \label{fig:curve_fitting}
\end{figure}

\textbf{Discontinuities.} Implicit models are able to approximate discontinuities sharply without introducing intermediate artifacts (Fig.~\ref{fig:curve_fitting}a), whereas explicit models (Fig.~\ref{fig:curve_fitting}d),
because they fit a continuous function to the data, take every intermediate value between training samples.
As the frequency of discontinuities increases, the implicit model predictions remain sharp at discontinuities, while also respecting local continuities, and with piece-wise linear extrapolations up to some decision boundary between training examples (Fig.~\ref{fig:curve_fitting}a-c). The explicit model interpolates across each discontinuity (Fig.~\ref{fig:curve_fitting}d-f). Once the training data is uncorrelated (i.e. random noise) and without regularization (Fig.~\ref{fig:curve_fitting}c, Fig.~\ref{fig:curve_fitting}f), implicit models exhibit a nearest-neighbors-like behavior, though with non-zero $\frac{\partial y}{\partial x}$ segments around each sample.

\textbf{Extrapolation.} For extrapolation outside the convex hull of the training data (Fig.~\ref{fig:curve_fitting}a-f), even with discontinuous or multi-valued functions, implicit models often perform piecewise linear extrapolation of the piecewise linear portion of the model nearest to the edge of the training data domain. Recent work \cite{xu2020neural} has shown that explicit models tend to perform linear extrapolation, but the analysis assumes the ground truth function is continuous.  \pete{Might want to mention here the score matching observation.}

\textbf{Multi-valued functions.} Instead of using $\arg\min$ to identify a single optimal value, $\arg\min$ may return a set of values, which may either be interpreted probabilistically as sampling likely values from the distribution, or in optimization as the {\em{set}} of minimizers ($\arg\min$ is set-valued).
\johnny{isn't argmin set-valued only if all the argmins are truly equally minimizing (e.g. cos(x))? But, that's not the case in these examples.  In practice, there is a single most minimum during evaluation}
Fig.~\ref{fig:curve_fitting_multivalued} compares a ReLU-MLP trained as a Mixture Density Network (MDN) vs an EBM across three example multi-valued functions.
\pete{We need a conclusion statement here too.}\johnny{proposal: The implicit model fits training the data more cleanly with less spurious predictions, and adapts to the number of modes without hyper parameters (Fig. 3a,b,d,e).  It is also able to maintain sharp boundaries between uniformly high and low probability regions of samples. (Fig 3c,f)}
%
%
\begin{figure}
\vspace{-1.0em}
\centering
\hspace*{-0.5cm} 
  \includegraphics[width=1.10\textwidth]{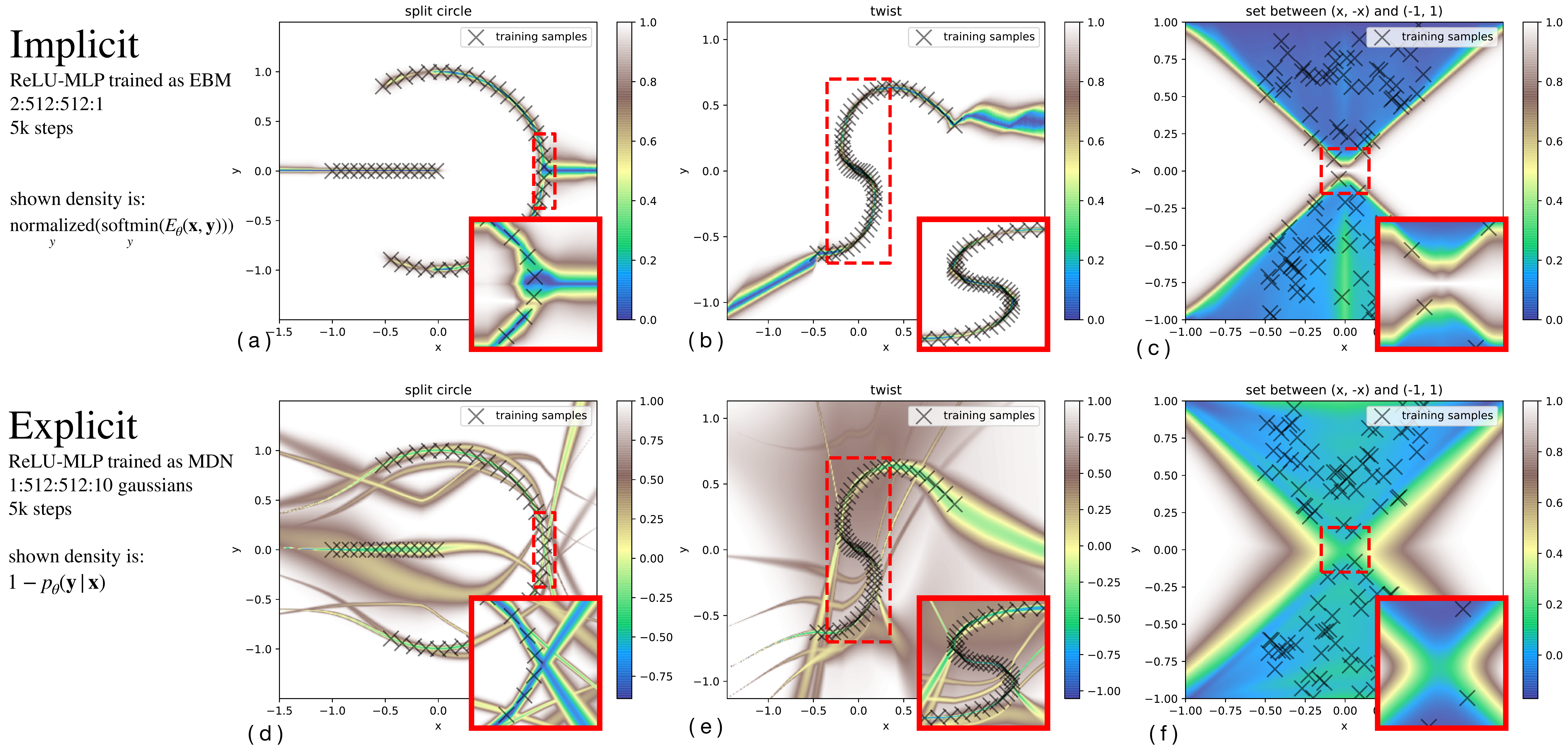}
  \vspace{-0.5cm}
  \caption{Representations of multi-valued functions showing extrapolations beyond the training samples (outside of shown `X' training samples) and detail views of notable regions. (a,d) Split circle with discontinuities and mode count changes; (b,e) locally continuous curve exhibiting hysteretic behavior, (c,f) set function of disjoint uniformly valid ranges.}
  \label{fig:curve_fitting_multivalued}
  \vspace{-1.0em}
\end{figure}

\textbf{Visual Generalization} Of particular relevance to learning visuomotor policies, we also find striking differences in extrapolation ability with converting high-dimensional image inputs into continuous outputs.  Fig.~\ref{fig:coord_regression} shows how on a simple visual coordinate regression task, which is a notoriously hard problem for convolutional networks \cite{liu2018intriguing}, an MSE-trained Conv-MLP model \cite{levine2016end} with CoordConv \cite{liu2018intriguing} struggles to extrapolate outside the convex hull of the training data. This is consistent with findings in \cite{florence2019self,zeng2020transporter}.  A Conv-MLP trained via late fusion (Fig.~\ref{fig:coord_regression}b) as an EBM, on the other hand, extrapolates well with only a few training data samples, achieving 1 to 2 orders of magnitude lower test-set error in the low-data regime (Fig.~\ref{fig:coord_regression}d). This is additional evidence that distinguishes implicit models from explicit models in a distinct way from multi-modality, which is absent in this experiment.  

\begin{figure}[H]
\centering
\hspace*{-0.3cm} 
  \includegraphics[width=1.03\textwidth]{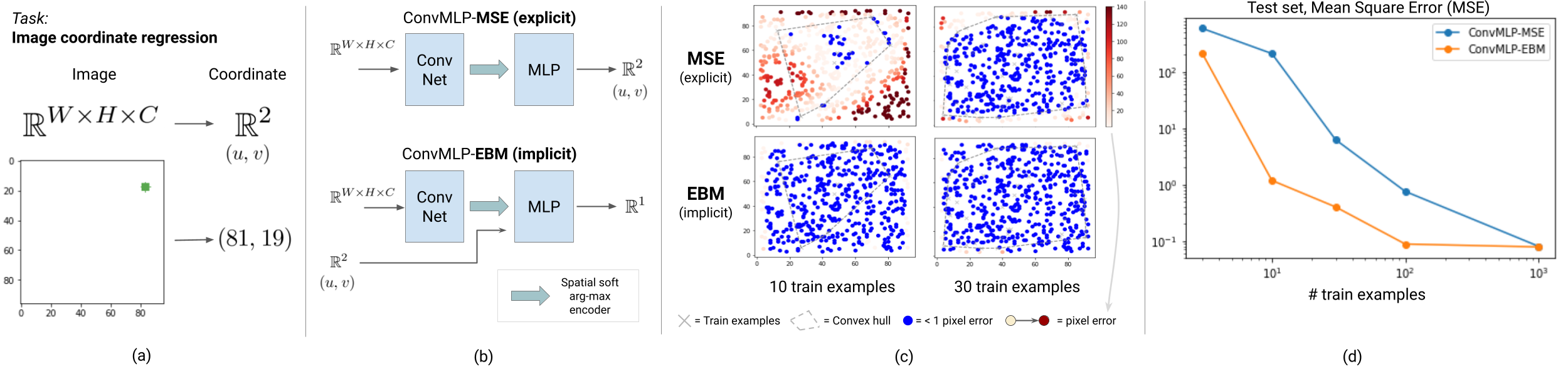}
  \caption{Comparison of implicit and explicit ConvMLP models on a simple coordinate regression task \cite{liu2018intriguing}, $\mathbb{R}^{W \times H \times C} \rightarrow \mathbb{R}^2$ (a).  The architectures shown in (b) are trained on images (example in a) to regress the $(u,v)$ coordinate of a green few-pixel dot.  The {\em{spatial generalization}} plot (c) shows the convex hull (gray dotted line) of the training data and shows that with only 10 training examples, the MSE-trained models struggle both to interpolate and extrapolate (c, top left).  At 30 train examples (c, top right) it can reasonably interpolate, but still struggles with extrapolation. ConvMLP-EBM, instead (c,bottom) performs well with little data, with 1 to 2 orders of magnitude lower test-set MSE loss (d) in the low-data regime.}
  \label{fig:coord_regression}
  \vspace{-0em}
\end{figure}




\section{Policy Learning Results}\label{sec:policy-results}
\begin{figure}[H]
\begin{minipage}{1.0\textwidth}
\centering
\includegraphics[width=1.0\textwidth]{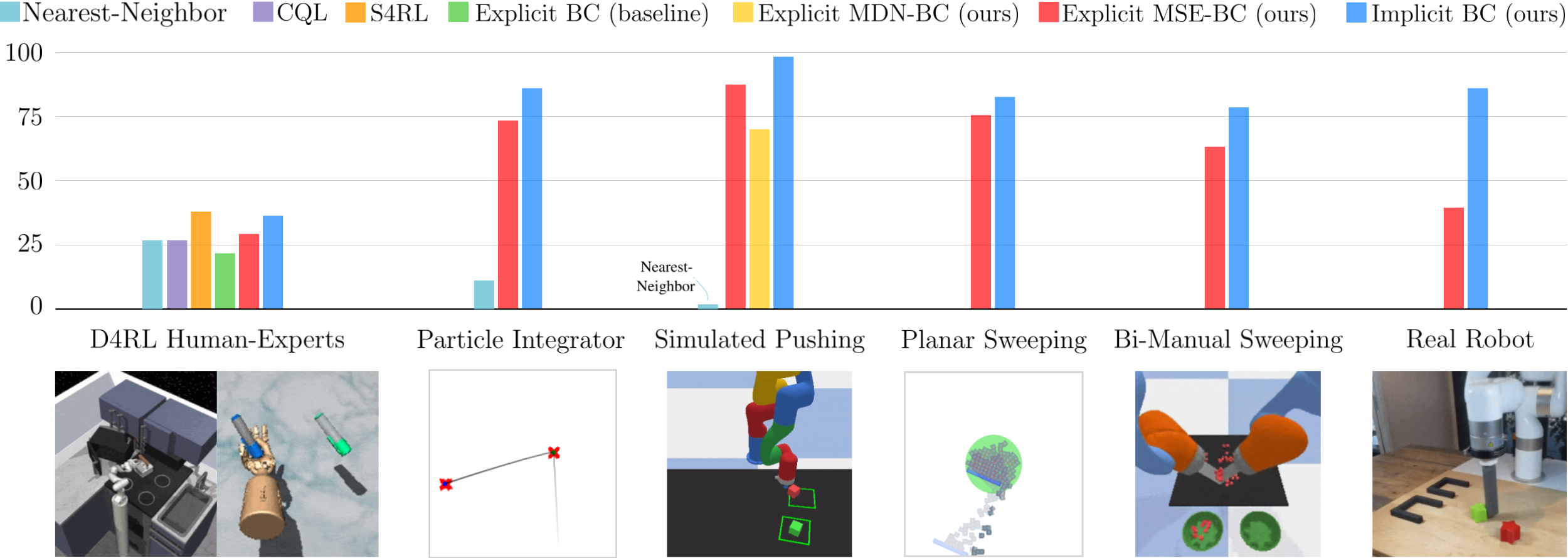}
\caption{Comparisons between implicit and explicit policies across 6 various simulated and real domains (Table \ref{table:task-attributes}), including author-reported baselines on the human-expert D4RL tasks.  See Appendix for full experimental protocol.  Standard deviations are shown in Tables 2, 3, 4, 5, 6.}\label{fig:policy_learning_tasks_bar_chart}
\end{minipage}
\vspace{-1.5em}
\end{figure}

\begin{wraptable}{r}{0.46\textwidth}
  \setlength\tabcolsep{2.3pt}
  \vspace{-0.8em}
  \centering
  \scriptsize
  \begin{tabular}{@{}lccccc}
  \toprule
  & image & human & unknown & multimodal \\
  Benchmark & input & demos & cardinality & solutions \\
  \midrule
  D4RL Human-Experts    & \xmark & \cmark & \xmark & \xmark \\
  Particle Integrator   & \xmark & \xmark & \xmark & \xmark \\
  Block Pushing         & \cmark & \xmark & \xmark & \cmark \\
  Planar Sweeping       & \cmark & \cmark & \cmark & \cmark \\
  Bi-Manual Sweeping    & \cmark & \xmark & \cmark & \cmark \\
  Real Robot            & \cmark & \cmark & \xmark & \cmark \\
  \bottomrule
  \end{tabular}
  \vspace{0.1em}
  \caption{\scriptsize Each benchmark is characterized by a unique set of attributes.}
  \vspace{-2.0em}
  \label{table:task-attributes}
\end{wraptable}

We evaluate implicit models for learning BC policies across a variety of robotic task domains (Fig. \ref{fig:policy_learning_tasks_bar_chart}). 
The goals of our experiments are three-fold: (i) to compare the performance of otherwise-identical policies when represented as either implicit or explicit models
, (ii) to test how well our models (both implicit and explicit) compare with author-reported baselines on a standard set of tasks, and (iii) to demonstrate that implicit models can be used to learn effective policies from human demonstrations with visual observations on a real robot. The following results and discussions are organized by task domain -- each evaluating a unique set of desired properties for policy learning (Table \ref{table:task-attributes}). All tasks are characterized by discontinuities and require generalization (\eg extrapolation) to some degree. 



\textbf{D4RL \cite{fu2020d4rl}} is a recent benchmark for offline reinforcement learning. We evaluate our implicit (EBM) and explicit (MSE) policies across the subset of tasks for which offline datasets of human demonstrations are provided, which is arguably is the hardest set of tasks. 
Surprisingly, we find that our implementations of both implicit and explicit policies significantly outperform the BC baselines reported on the benchmark, and provide competitive results with state-of-the-art offline reinforcement learning results reported thus far, including CQL \cite{kumar2020conservative} and S4RL \cite{sinha2021s4rl}.  By adding perhaps the simplest way to use reward information, if we prioritize sampling to be only the top 50\% of demonstrations sorted by their returns (similar to Reward-Weighted Regression (RWR) \cite{peters2007reinforcement}), this intriguingly generally improves implicit policies, in some cases to new state-of-the-art performance, while less so for explicit models.  This suggests that implicit BC policies value data quality higher than explicit BC policies do.  A simple Nearest-Neighbor baseline (see Appendix) performs better than one might expect on these tasks, but on average not as well as implicit BC.  

\begin{table}
  \centering
  \tiny
  \setlength\tabcolsep{5.3pt}
  \begin{tabular}{@{}llcccccccccc@{}}
  \toprule
    & & \multicolumn{4}{c}{Baselines} & \multicolumn{4}{c}{Ours} \\
    \cmidrule(lr){3-6} \cmidrule(lr){7-10}
         &  & &              &     &      & {\em{Explicit}} & {\em{Implicit}} & {\em{Explicit}} & {\em{Implicit}}  \\
  Method &  & Nearest- & BC           & CQL \cite{kumar2020conservative} & S4RL \cite{sinha2021s4rl} & BC (MSE) & BC (EBM) & BC (MSE) & BC (EBM)  \\
         &  & Neighbor & (from CQL \cite{kumar2020conservative})   &     &      &        &        & w/ RWR \cite{peters2007reinforcement} & w/ RWR \cite{peters2007reinforcement} \\
  \midrule
  Uses data & & $(\mathbf{o}, \mathbf{a})$& $(\mathbf{o}, \mathbf{a})$ & $(\mathbf{o}, \mathbf{a}, r)$ & $(\mathbf{o}, \mathbf{a}, r)$ & $(\mathbf{o}, \mathbf{a})$ & $(\mathbf{o}, \mathbf{a})$ & $(\mathbf{o}, \mathbf{a}, r)$ & $(\mathbf{o}, \mathbf{a}, r)$\\
  \midrule
  \emph{Domain} & \emph{Task Name} \\
  \hline
  \hline
  \multirow{3}{*}{Franka} & kitchen-complete & 1.92 $\pm$0.00 & 1.4 & 1.8 & 3.08 & 1.76 $\pm$0.04 & \textbf{3.37} $\pm$0.19 & 1.22 $\pm$0.18 & \textbf{3.37} $\pm$0.01 \\  
                          & kitchen-partial  & 1.70 $\pm$0.00 & 1.4 & 1.9 & \textbf{2.99} & 1.69 $\pm$0.02 &  1.45 $\pm$0.35 & 1.86 $\pm$0.26 & 2.18 $\pm$0.05 \\
                          & kitchen-mixed    & 1.46 $\pm$0.00 & 1.9 & 2.0 &               & \textbf{2.15} $\pm$\textbf{0.06} &  1.51 $\pm$0.39 & 2.03 $\pm$0.06 & \textbf{2.25} $\pm$\textbf{0.14} \\
  \hline
  \hline
  \multirow{5}{*}{Adroit} & pen-human           & 1908.0 $\pm$0.0 & 1121.9 & 1214.0 & 1419.6 & 2141 $\pm$109 & \textbf{2586} $\pm$\textbf{65} & 2108 $\pm$58.8 & \textbf{2446} $\pm$\textbf{207}\\
                          & hammer-human        & -85.2  $\pm$0.0 & -82.4  & 300.2  &  \textbf{496.2} & -38 $\pm$25 & -133 $\pm$26  & -35.1 $\pm$45.1 & -9.3  $\pm$45.5\\
                          & door-human          & 91.8  $\pm$0.0 & -41.7  & 234.3  &  \textbf{736.5} & 79 $\pm$15 & 361 $\pm$67  & 17.9 $\pm$ 13.8 & 399 $\pm$34\\
                          & relocate-human      & -3.8 $\pm$0.0 & -5.6   & 2.0    &    2.1 & -3.5 $\pm$1.1 & -0.1 $\pm$2.4 & -3.7 $\pm$0.3 & $\textbf{3.6} \pm\textbf{2.5}$\\
  \bottomrule
  \end{tabular}
  \caption{\scriptsize Baseline comparisons on D4RL \cite{fu2020d4rl} tasks with human-expert data.  Results shown are the average of 3 random seeds, 100 evaluations each, with $\pm$ std. dev.  Baselines from \cite{kumar2020conservative} and \cite{sinha2021s4rl} didn't report standard deviations. See Appendix for more on experimental protocol.}
  \label{table:d4rl-table}
  \vspace{-2.5em}
\end{table}

While many of the D4RL tasks have complex high-dimensional action spaces (up to 30-D), they do not emphasize the full spectrum of task attributes (Table~\ref{table:task-attributes}) we are interested in.  The following tasks isolate other attributes or introduce new ones, such as highly stochastic dynamics (i.e., single-point-of-contact block pushing), complex multi-object interactions (many small particles), and combinatorial complexity.

\begin{wrapfigure}{r}{0.3\textwidth}
  \vspace{-1.8em}
  \begin{center}
    \includegraphics[width=0.3\textwidth]{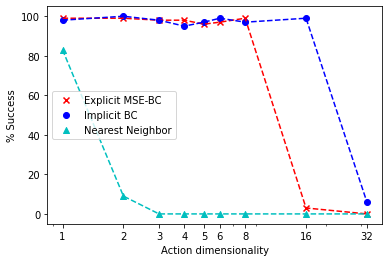}
  \end{center}
  \vspace{-1em}
  \caption{Comparison of policy performance on the $N$-D particle environment,  2,000 demonstrations each.}
  \label{fig:particle-results}
  \vspace{-1.0em}
\end{wrapfigure}

\textbf{N-D Particle Integrator} is a simple environment with linear dynamics but where a discontinuous oracle policy is used to generate training demonstrations: once within the vicinity of goal-conditioned location (Fig.~\ref{fig:policy_learning_tasks_bar_chart}, shown for $N=2$), the policy must switch to the second goal.  The benefit of studying this environment is two-fold: (i) it has none of the complicating attributes in Table~\ref{table:task-attributes} and so allows us to study discontinuities in isolation, and (ii) we can define this simple environment to be in $N$ dimensions.  Varying $N$ from 1 to 32 dimensions, but holding the number of demonstrations constant, we find we are able to train 95\% successful implicit policies up to 16 dimensions, whereas explicit (MSE) policies can only do 8 dimensions with the same success rate.  The Nearest-Neighbor baseline, meanwhile, cannot generalize, and only performs well on the 1D task (see Appendix for more analysis).


\begin{wraptable}{r}{0.40\textwidth}
  \setlength\tabcolsep{2.3pt}
  \vspace{-0.8em}
  \centering
  \tiny
  \begin{tabular}[b]{@{}lccc@{}}
  \toprule
  Method & Single Target, & Multi Target, & Single Target, \\
         & states         & states        & pixels \\
  \midrule
  EBM                & \textbf{100} $\pm$\textbf{0} & 99.0 $\pm$0.0 & \textbf{100} $\pm$\textbf{0}\\
  MDN                & \textbf{100} $\pm$\textbf{0} & \textbf{99.7} $\pm$\textbf{0.5} & 10.0 $\pm$4.3\\
  MSE                & 98.3 $\pm$0.5 & 89.7 $\pm$4.8 & 87.0 $\pm$4.1\\
  Nearest-Neighbor   & 4.0 $\pm$0.0 & 0.0 $\pm$0.0 & 4.3 $\pm$1.9 \\
 
  \bottomrule
  \end{tabular}
  \vspace{0.3em}
  \caption{\scriptsize Results on simulated xArm6 pushing tasks, average of 3 random seeds, 100 evaluations each, with $\pm$ std. dev.}
  \vspace{-2.0em}
  \label{table:sim-pushing}
\end{wraptable}

\textbf{Simulated Pushing} consists of a simulated 6DoF robot xArm6 in PyBullet \cite{coumans2016pybullet} equipped with a small cylindrical end effector. The task is to push a block into the target goal zone, marked by a green square labeled on the tabletop. We investigate 2 variants: (a) pushing a single block to a single target zone, or (b) also pushing the block to a second goal zone (multistage). We evaluate implicit (EBM) and explicit (MSE and MDN \cite{rahmatizadeh2018vision,ha2018world}) policies on both variants, trained from a dataset of 2,000 demonstrations using a scripted policy that readjusts its pushing direction if the block slips from the end effector. Results in Table~\ref{table:sim-pushing} show that all learning methods perform well on the single-target task, while MSE struggles with the slightly longer task horizon.  For the image-based task, the MDN significantly struggles compared to MSE and EBM.  The failures of the Nearest-Neighbor baseline, with only 0-4\% success rate, show that generalization is required for this task.


\textbf{Planar Sweeping} \cite{suh2020surprising} is a 2D environment that consists of an agent (in the form of a blue stick) where the task is to push a pile of 50 - 100 randomly positioned particles into a green goal zone. The agent has 3 degrees of freedom (2 for position, 1 for orientation). We train implicit (EBM) and explicit (MSE) policies from 50 teleoperated human demonstrations, and test on episodes with unseen particle configurations. For the image-based inputs, we also test two types of encoders with different forms of dimensionality reduction: spatial soft(arg)max and average pooling over dense features (see Appendix for architecture descriptions). For the state-based inputs, since the number of particles vary between episodes, we flatten the poses of the particles and 0-pad the vector to match the size of the vector at maximum particle cardinality.

\begin{wrapfigure}{r}{0.64\textwidth}
  \vspace{-1.3em}
  \centering
  \includegraphics[width=0.215\textwidth]{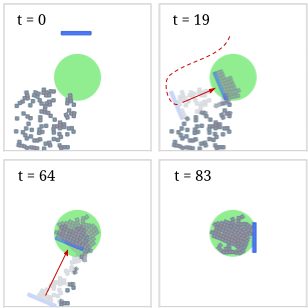}
  \quad
  \setlength\tabcolsep{2.3pt}
  \tiny
  \begin{tabular}[b]{@{}lcccccccc@{}}
  \toprule
  & & \multicolumn{3}{c}{\# ResNet layers} \\
  \cmidrule(lr){3-5}
  Method & Input \& Encoder & 8 & 14 & 20 \\
  \midrule
  EBM                & image + softmax & 78.7 $\pm$4.9 & 82.1 $\pm$0.9 & \textbf{82.6} $\pm$\textbf{3.1} \\
  EBM                & image + pool & 78.0 $\pm$2.2 & 76.5 $\pm$1.0 & 74.2 $\pm$1.9 \\
  EBM                & state & 28.7 $\pm$0.8 & 29.2 $\pm$0.5 & 28.9 $\pm$0.2 \\
  \midrule
  MSE                & image + softmax & 62.9 $\pm$5.0 & 51.4 $\pm$8.9 & 56.6 $\pm$5.2 \\
  MSE                & image + pool & 75.6 $\pm$1.3 & 73.9 $\pm$1.7 & 74.8 $\pm$1.2 \\
  MSE                & state & 28.9 $\pm$0.2 & 28.2 $\pm$0.4 & 27.8 $\pm$0.3 \\
      \bottomrule
  \end{tabular}
  \captionlistentry[table]{A table beside a figure}
  \captionsetup{labelformat=andtable}
  \caption{Image-based implicit (EBM) policies outperform explicit (MSE) ones in learning to control the agent (blue) to sweep an unknown number of particles (gray) into a target goal zone (green).  Trained on 50 human demonstrations.}
  \vspace{-2em}
\end{wrapfigure}


The results in Table 4 (averaged over 3 training runs with different seeds) suggest that image-based EBMs outperform the best MSE architectures by 7\%. Interestingly, image-based EBMs seem to synergize well with spatial soft(arg)max for dimensionality reduction, as opposed to pooling, which works best for MSE explicit policies. 
In both cases, state observations as inputs do not perform well compared with image pixel inputs. This is likely because the particles have symmetries in image space, but not when observed as a vector of poses.

\textbf{Simulated Bi-Manual Sweeping} consists of two robot KUKA IIWA arms equipped with spatula-like end effectors. The task is to scoop up randomly configured particles from a $0.4m^2$ workspace and transport them into two bowls, which should be filled up equally. Successfully scooping particles and transporting them requires precise coordination between the two arms (\eg such that the particles do not drop while being transported to the bowls). The action space is 12DoF (6DoF Cartesian per arm), and each episode consists of 700 steps recorded at 10Hz. Perspective RGB images from a simulated camera are used as visual input, along with current end effector poses as state input. The task is characterized by many mode changes and discontinuities (transitioning from scooping to lifting, from lifting to transporting, and deciding which bowl to transport to). EBM and MSE policies on the task use the best corresponding image encoder from the planar sweeping task. As shown in Table 5, our results show that EBM outperforms MSE by 14\%.

\begin{figure}[!ht]
  \centering
  \includegraphics[width=0.65\textwidth]{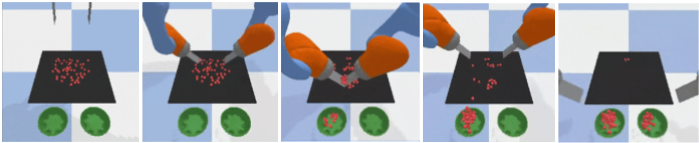}
  \quad
  \setlength\tabcolsep{2.3pt}
  \scriptsize
  \begin{tabular}[b]{@{}lccc@{}}
  \toprule
  Method & Input and Encoder & Success \% \\
  \midrule
  EBM                & image + softmax & \textbf{78.2} $\pm$2.7 \\
  MSE                & image + pool & 63.9 $\pm$7.7 \\
  \bottomrule
  \end{tabular}
  \captionlistentry[table]{A table beside a figure}
  \captionsetup{labelformat=andtable}
  \caption{Image-based implicit (EBM) policies outperform explicit (MSE) ones in learning to control two robot arms (6DoF + 6DoF) with spatula-like end effectors to scoop up particles (red) from a workspace and equally distribute them across two bowls (green). Success \% is the average ratio of particles successfully moved into the bowls across 10 rollouts over 3 different model seeds. Trained on 1,000 scripted demonstrations.}
  \vspace{-1em}

\end{figure}


\textbf{Real Robot Manipulation}, using a cylindrical end-effector on an xArm6 robot (Fig.~\ref{fig:hardware}a), we evaluate implicit BC and explicit BC policies on 4 real-world manipulation pushing tasks: 1) pushing a red block and a green block into assigned target fixtures, 2) pushing the red and green blocks into either target fixture, in either order, 3) precise pushing and insertion of a blue block into a tight (1mm tolerance) target fixture, and 4) sortation of 4 blue blocks and 4 yellow into different targets. The observation input is only raw perspective RGB images at 5Hz, with task horizons up to 60 seconds, and teleoperated demonstrations.


\begin{table}[H]
  \centering
  \scriptsize
  \begin{tabular}{@{}lcccccccccc@{}}
  \toprule
    Task & \multicolumn{1}{c}{Push-Red-then-Green} & \multicolumn{1}{c}{Push-Red/Green-Multimodal}  & \multicolumn{1}{c}{Insert-Blue} & \multicolumn{1}{c}{Sort-Blue-from-Yellow}  \\
     \cmidrule(lr){2-2} \cmidrule(lr){3-3} \cmidrule(lr){4-4} \cmidrule(lr){5-5}
  \# demos &  95  & 410 & 223 & 502 \\ \hline 
  Avg. lengths $\pm$ std. & 19.1 $\pm$2.5  & 19.0 $\pm$3.1  & 22.1 $\pm$5.5  & 45.2 $\pm$8.2  \\ 
  {[min, max] (seconds)} & [14.2, 25.1] & [11.8, 28.1] & [13.0, 43.5] & [25.8, 60.5] \\
  \midrule
  Success criterion & 1.0 if both blocks in target & 1.0 if both blocks in target        & \begin{tabular}[c]{@{}r@{}}0.5 for partial insert\\ 1.0 for full insert\end{tabular} & $\frac{1}{8}$ for each correct block in target        \\ \hline
  \emph{Success avg. (\%)}\\
Implicit BC (EBM) & \textbf{85.0} $\pm$\textbf{5.0}        & \textbf{88.3} $\pm$\textbf{7.6}               & \textbf{83.3} $\pm$\textbf{3.8}                                                                & \textbf{48.3} $\pm$\textbf{4.6}       \\   
Explicit BC (MSE) & 35.0 $\pm$18.0                & 55.0 $\pm$18.0                       & 6.7 $\pm$9.4                                                                          & 19.6 $\pm$1.5                              \\
  \bottomrule
  \end{tabular}
  \caption{\scriptsize Real-world robot results, success \% shown is mean +/- std.dev (20 rollouts per seed, 3 seeds = 60 trials per method per task).}
  \label{table:real-hardware-table}
\end{table}

\begin{figure}
\centering
  \vspace{-1em}
   \includegraphics[width=1.0\textwidth]{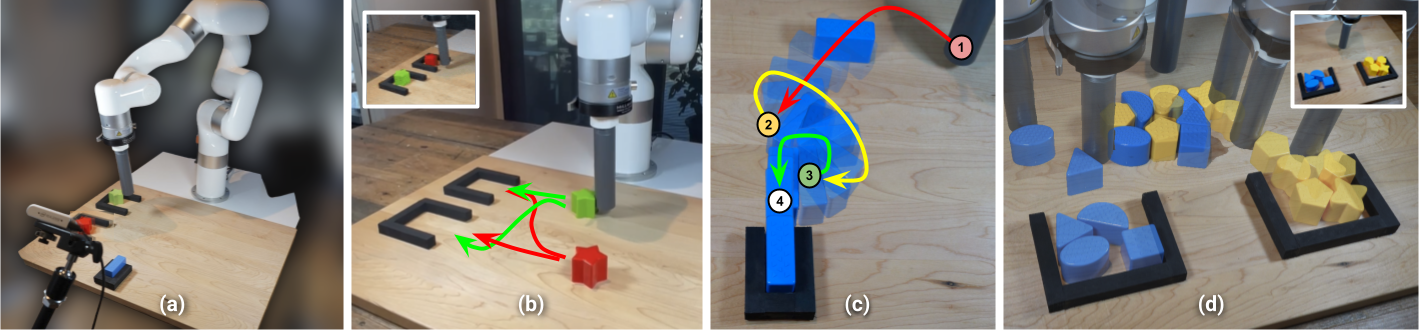}
  \caption{Results using our hardware configuration (a, see Appendix for full description) on real-world visual manipulation tasks, including (b) multi-modal targeted block pushing, (c) precise oriented insertion requiring 1mm precision, and (d) a combinatorially complex sorting task.}
  \label{fig:hardware}
  \vspace{-1em}
\end{figure}

Across all four tasks, we observe significantly higher performance for the implicit policies compared to the explicit baseline.  This is especially apparent on the pushing-and-oriented-insertion task ({\em{Insert Blue}}), which requires highly discontinuous behavior in order to subtly nudge enough, but not too far, the block into place (Fig.~\ref{fig:hardware}c). On this task we see the implicit BC policy has an {\em{order of magnitude}} higher success rate than the explicit BC policy.  The sorting task in particular ({\em{Sort-Blue-From-Yellow}}, Fig.~\ref{fig:hardware}d) is our attempt to push the generalization abilities of our models, and we see a 2.4x higher success rate for the implicit policy.  Note these experimental results are averaged over 3 different models, for each task, for each policy type.  The red/green pushing tasks, including multi-modal variant (Fig.~\ref{fig:hardware}b), also show notably higher success rates for the implicit policies.  These real-world results are best appreciated in our video.




\section{Theoretical Insight: Universal Approximation with Implicit Models}\label{sec:theory}

In previous sections, we have empirically demonstrated the ability of implicit models to handle discontinuities (Section~\ref{sec:empirical-properties}), and we hypothesized this is one of the reasons for the strong performance of implicit BC policies (Section~\ref{sec:policy-results}).  Two theoretical questions we now ask are: (i) is there a provable notion for \emph{what class of functions} can be represented by implicit models given some analytical $E(\cdot)$, and (ii) given that energy functions learned from data may always be expected to have non-zero error of approximating any function, are there inference risks with large behaviour shifts resulting from a combination of $\arg\min$ and spurious peaks in $E(\cdot)$? Recent work \cite{marx2021semi} has shown that a large class of functions (namely, functions defined by finitely many polynomial inequalities) can be approximated implicitly by $\arg\min_\bfy g(\bfx, \bfy)$ using SOS polynomials to represent $g(\cdot)$.  Here we show that for implicit models with $g_{\theta}$ represented by any continuous function approximator (such as a deep ReLU-MLP network), $\arg\min_\bfy \ g_{\theta}(\bfx, \bfy)$ can represent a larger set of functions including multi-valued functions and discontinuous functions (Thm.~\ref{universal-rep}), to arbitrary accuracy (Thm.~\ref{universal-approximation}).  These results are stated formally in the following; proofs are in the Appendix.

\begin{Theorem}\label{universal-rep}
For any set-valued function $F(\bfx):\ \bfx \in \Rm \rightarrow P(\mathbb{R}^n) \setminus \{ \emptyset \}$ where the graph of $F$ is closed,
there exists a continuous function $g(\bfx, \bfy): \ \Rmn \rightarrow \Rone$,
such that $\underset{\bfy}{\arg\min} \ g(\bfx, \bfy) = F(\bfx)$ for all $\bfx$.
\end{Theorem}




\begin{figure}[H]
\centering
  \vspace{-1em}
  \includegraphics[width=1.0\textwidth]{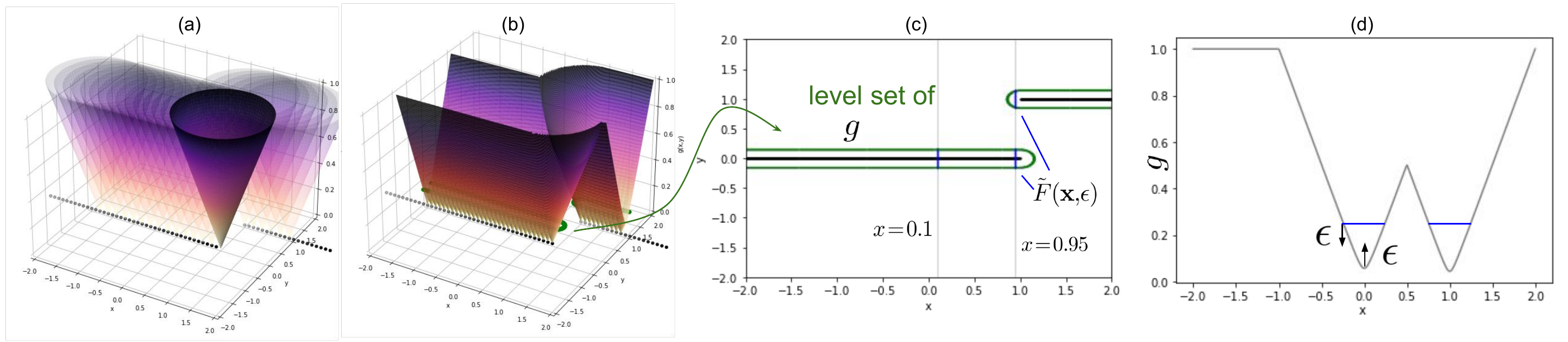}
  \caption{Visual explanation of the results presented in  Thms.~\ref{universal-rep} and Thms.~\ref{universal-approximation}, the construction of a continuous function $g(x,y)$ for which $\text{argmin}_y \ g(x,y)$ yields $f(x) = \{\{1,0\} \  \text{if}  \ x = 1, 1\  \text{if} \ x > 1, \ 0 \ \text{otherwise}\}$.  The function $g(\cdot)$ (b) is the minimum distance to the graph of $f()$, for example the infimum over a set of cones (a). The approximation guarantee (Thm.~\ref{universal-approximation}) can be visualized via the level-sets of $g(\cdot)$ (b,c), and a slice (d) of $g(\cdot)$.  For more explanation, see the Appendix.}.
  \label{fig:inf-over-cones}
  \vspace{-1em}
\end{figure}
%
%
\begin{Theorem}\label{universal-approximation}
For any set-valued function $F(\bfx):\ \bfx \in \Rm \rightarrow P(\mathbb{R}^n) \setminus \{ \emptyset \}$, there exists a function $g(\cdot)$ that can be approximated by some continuous function approximator $g_{\theta}(\cdot)$ with arbitrarily small bounded error $\epsilon$,
such that $\hat{\bfy} 
= \underset{\bfy}{\text{argmin}} \ g_{\theta}(\bfx, \bfy)$ provides the guarantee that the distance from 
$(\bfx, \hat\bfy)$ to the graph of $F$ is less than $\epsilon$.
\end{Theorem}

Of practical note, explicit functions ($F(\bfx)$ in Thms.~\ref{universal-rep} and ~\ref{universal-approximation})
with arbitrarily small or large Lipschitz constants can be approximated by an implicit function with bounded Lipschitz constant (see Appendix for more discussion).
This means that implicit functions can approximate steep or discontinuous explicit functions without large gradients in the function approximator that may cause generalization issues.  This is not the case for explicit continuous function approximators, which must match the large gradient of the approximated function.
In both their multi-valued nature and discontinuity-handling, the approximation capabilities of implicit models are distinctly superior to explicit models.
See Fig.~\ref{fig:inf-over-cones} for visual intuition, and more discussion in the Appendix.  






\section{Related Work}\label{sec:related}

\textbf{Energy-Based Models, Implicit Learning.}  Reviews of energy-based models can be found in LeCun et al. \cite{lecun2006tutorial} and Song \& Kingma \cite{song2021train}.  Du \& Mordatch \cite{du2019implicit} proposed Langevin MCMC \cite{welling2011bayesian} sampling for training and implicit inference, and argued for several strengths of implicit generation, including compositionality and empirical results such as out-of-distribution generalization and long-horizon sequential prediction.  A general framework for energy-based learning of behaviors is also presented in \cite{mordatch2018concept}.  In applications, energy based models have recently shown state-of-the-art results across a number of domains, including various computer vision tasks \cite{ gustafsson2020energy,gustafsson2020train}, as well as generative modeling tasks such as image and text generation \cite{du2019implicit,du2020improved, deng2020residual}.
Many other works have investigated using the notion of implicit functions in learning, including works that investigate implicit layers \cite{amos2017optnet, niculae2018sparsemap, wang2019satnet, bai2019deep}. There is also a surge of interest in geometry representation learning in implicit representations \cite{park2019deepsdf,mescheder2019occupancy,chen2019learning,saito2019pifu}. In robotics, implicit models have been developed for modeling discontinuous contact dynamics \cite{pfrommer2020contactnets}.

\igor{drop in a textbook (Russ'?) ref about switching dynamical systems?} 

\textbf{Energy-Based Models in Policy Learning}. In reinforcement learning, \cite{haarnoja2017reinforcement} uses an EBM formulation as the policy representation.  Other recent work \cite{du2020planning} uses EBMs in a model-based planning framework, or uses EBMs in imitation learning \cite{liu2020energy} but with an on-policy algorithm.  A trend as well in recent RL works has been to utilize an EBM as part of an overall algorithm, i.e. \cite{kostrikov2021offline,nachum2021provable}.)



\textbf{Policy Learning via Imitation Learning.} In addition to behavioral cloning (BC) \cite{pomerleau1989alvinn}, the machine learning and robotics communities have explored many additional approaches in imitation learning \cite{osa2018algorithmic,peng2018deepmimic,peng2021amp}, often in ways that need additional information. One route is by collecting on-policy data of the learned policy, and potentially either labeling with rewards to perform on-policy reinforcement learning (RL) \cite{atkeson1997robot, ng2000algorithms, rajeswaran2017learning} or labeling actions by an expert \cite{ross2011reduction}.  Distribution-matching algorithms like GAIL \cite{ho2016generative} require no labeling, but may require millions of on-policy environment interactions. While algorithms like ValueDice~\cite{kostrikov2019imitation} implement distribution matching in a sample-efficient off-policy setting, they have not been proven on image-observations or high degree-of-freedom action spaces.  Another route to using more information beyond BC is for the off-policy data to be labeled with rewards, which is the focus of the offline RL community \cite{fu2020d4rl}.  All of these directions are good ideas. A perhaps not fully appreciated finding, however, is that in some cases even the simplest forms of BC can yield surprisingly good results.  On offline RL benchmarks, prior works' implementations of BC already show reasonably competitive results with offline RL algorithms \cite{fu2020d4rl, gulcehre2020rl}. 
In real-world robotics research, BC has been widely used in policy learning \cite{zhang2018deep,rahmatizadeh2018vision,florence2019self,zeng2020transporter}.  Perhaps the success of BC comes from its \emph{simplicity}: it has the lowest data collection requirements (no reward labels or on-policy data required), can be data-efficient \cite{florence2019self, zeng2020transporter}, and it is arguably the simplest to implement and easiest to tune (with fewer hyperparameters than RL-based methods).  

\textbf{Approximation of Discontinuous Functions.} The foundational results of Cybenko \cite{cybenko1989approximation} and others in Universal Approximation of neural networks have had foundational impact in guiding machine learning research and applications. Various approaches have been developed in the function approximation literature and elsewhere to approximate discontinuous functions \cite{butzer1987approximation,tampos2012accurate,kvernadze2010approximation,stella2016very}, which typically do not use neural networks. Also motivated by applications to modeling phenomena for robots, \cite{selmic2002neural} develops theory of approximating discontinuous functions with neural networks, but the method requires a-priori knowledge of the discontinuity's location.  Our work builds on the well-known and well-applied results in continuous neural networks, but through composition with $\arg\min$ provides a notion of universal approximation even for discontinuous, set-valued functions.

\section{Conclusion}\label{sec:conclusion}


In this paper we showed that reformulating supervised imitation learning as a conditional energy-based modeling problem, with inference-time implicit regression, often greatly outperforms traditional explicit policy baselines. This includes on tasks with {\em{high-dimensional action spaces}} (up to 30-dimensional in the D4RL human-expert tasks), {\em{visual observations}}, and {\em{in the real world}}. In terms of limitations, a primary comparison with explicit models is that they typically require more compute, both in training and inference (see Appendix for comparisons).  However, we have both shown that we can run implicit policies for real-time vision-based control in the real world, and training time is modest compared to offline RL algorithms.  To further motivate the use of implicit models, we presented an intuitive analysis of energy-based model characteristics, highlighting a number of potential benefits that, to the best of our knowledge, are not discussed in the literature, including their ability to accurately model discontinuities. Lastly, to ground our results theoretically we developed a notion of universal approximation for implicit models which is distinct from that of explicit models.

\clearpage
\acknowledgments{
The authors would like to thank Vikas Sindwhani for project direction advice; Steve Xu, Robert Baruch, Arnab Bose for robot software infrastructure; Jake Varley, Alexa Greenberg for ML infrastructure; and Kamyar Ghasemipour, Jon Barron, Eric Jang, Stephen Tu, Sumeet Singh, Jean-Jacques Slotine, Anirudha Majumdar, Vincent Vanhoucke for helpful feedback and discussions.
}



\small
\bibliography{main}  
\normalsize

\newpage

\appendix

\begin{center}
{\bf{\LARGE Appendix for {\em Implicit Behavioral Cloning}}}
\end{center}

{
  \hypersetup{linkcolor=black}
  \tableofcontents
}





\addtocontents{toc}{\protect\setcounter{tocdepth}{3}}
\section{Contributions Statement}\label{sec:appendix-contributions}

Due to space constraints we did not include a comprehensive contributions statement in the main manuscript, but include one here for clarity:

\begin{enumerate}
    \item We present Implicit Behavioral Cloning (Implicit BC), which is a novel, simple method for imitation learning in which behavioral cloning is cast as a conditional energy-based modeling (EBM) problem, and inference is performed via sampling-based or gradient-based optimization.
    \item We validate Implicit BC in real-world robot experiments, in which we demonstrate physical robots performing several end-to-end, contact-rich pushing tasks (including precision insertion, and multi-item sorting) driven with only images as input, and only human demonstrations provided as training data.  Implicit BC performs significantly better than our explicit BC baseline across all real-world tasks, including an {\em{order-of-magnitude}} increase in performance on the precision insertion task.  On the sorting task, the models are shown to be capable of solving an up-to-60-second horizon for a contact-rich, combinatorial task with complex multi-object collisions.
    \item We present extensive simulation experiments comparing Implicit BC to both comparable explicit models from the same codebase, and also author-reported quantitative results on the human-expert tasks from the standard D4RL benchmark.  We find both our explicit BC and implicit BC models provide competitive or state-of-the-art performance on D4RL tasks with human-provided demonstrations, despite using no reward information.  Averaged across all tasks, we find implicit BC outperforms our own best explicit BC models.
    \item We analyze the nature of implicit models in simple 1D-1D examples, and we highlight aspects of implicit models that we believe are not known to the generative modeling community, including their behavior (i) at discontinuities and (ii) in extrapolation. 
    \item We provide theoretical insight into implicit models, including proofs of their (i) representational abilities (Thm. \ref{thm:any-function-theorem}), and (ii) approximation abilities (Thm. \ref{universal-approximation-tmp}), which are shown to be distinct from continuous explicit models in their ability to handle discontinuities and set-valued functions. 
\end{enumerate}

\section{Energy-Based Model Training and Implicit Inference Details}\label{sec:ebm_variants_table}

Our results critically depend on energy-based model (EBM) training, but we do not consider the specific methods we use to be our main contributions (see Sec.~\ref{sec:appendix-contributions} for a list).  That said, after considerable experience training conditional EBMs on both simple function-fitting tasks, and on policy learning tasks, we believe it is useful to the research community to describe method specifics in detail.  Our goal is to emphasize simplicity when possible, in order to encourage more folks to use implicit energy-based regression rather than explicit regression.  We first review our approach using derivative-free optimization, then our  autoregressive version, and then our approach using Langevin gradient-based sampling.  For each, we discuss (i) how to {\em{train}} the models, and (ii) how to perform {\em{inference}} with the models.  For a more comprehensive overview of training EBMs, see \cite{song2021train}.  Note we will release code as well for training and inference.

For all methods, to compute $\bfy_{\text{min}}$ and $\bfy_{\text{max}}$ we (1) take the per-dimension min and max over the training data, (2) add a small buffer, typically 0.05($\bfy_{\text{max}}-\bfy_{\text{min}}$) on each side, and then (3) clip these min and max values to the environments' allowed min/max values.  For agents that do not use the full range of the environments' allowed values for a given dimension, this enables more precision on that action dimension.  Also all methods use Adam optimizer with default $\beta_1=0.9$, $\beta_2=0.999$ values.

\subsection{Method with Derivative-Free Optimization.}\label{subsec:dfo}

For training, this method is very simple.  For counter-examples we draw from the uniform random distribution:
${\color{red} \tilde{\mathbf{y}}} \sim \mathcal{U}(\mathbf{y}_{\text{min}}, \mathbf{y}_{\text{max}})$,
where $\mathbf{y}_{\text{min}}, \mathbf{y}_{\text{max}}\in \mathbb{R}^m$.  Training consists of drawing batches of data, sampling counterexamples for each sample in each batch, and applying $\mathcal{L}_{\text{InfoNCE}}$ (Sec.~2).  We typically use a batch size of 512, with 256 counter-examples per sample in the batch.  All $\{\bfx\}$ and $\{\bfy\}$ (i.e. $\bfo$ and $\bfa$ for observations and actions), in the training dataset are normalized to per-dimension zero-mean, unit variance. We use typically a $1e-3$ initial learning rate and an exponential decay, 0.99 decay each 100 steps.  We find that regularizing the models with Dropout does not help performance, perhaps because the stochastic training process (counter-example sampling in each training step) self-regularizes the models.

Given a trained energy model $E_{\theta}(\bfx, \bfy)$, we use the following derivative-free optimization algorithm to perform inference:

\begin{algorithm}[H]
\SetAlgoLined
\KwResult{$\hat{\bfy}$}
 Initialize: $\{\tilde{\bfy_i}\}_{i=1}^{N_{\text{samples}}} \sim \mathcal{U}(\mathbf{y}_{\text{min}}, \mathbf{y}_{\text{max}})$, $\sigma = \sigma_{\text{init}}$  \;
  \For{iter in 1, 2, ..., $N_{\text{iters}}$}{
    $\{E_i\}_{i=1}^{N_{\text{samples}}} = \{E_{\theta}(\bfx, \tilde{\bfy}_i)\}_{i}^{N_{\text{samples}}}$ \ {\scriptsize (compute energies)}\;
    $\{\tilde{p}_i\}_{i=1}^{N_{\text{samples}}} = \{\frac{e^{-E_i}}{\sum_{j=1}^{N_{\text{samples}}} e^{-E_j}} \} \ \ $ {\scriptsize (softmax)}\;
    
    \If{iter < $N_{\text{iters}}$}{
        $\{\tilde{\bfy_i}\}_{i=1}^{N_{\text{samples}}} \gets  \ \sim \text{Multinomial}(N_{\text{samples}},\{\tilde{p}_i\}_{i=1}^{N_{\text{samples}}},\{\tilde{\bfy_i}\}_{i=1}^{N_{\text{samples}}})$  {\scriptsize (resample with replacement)}\;
        $\{\tilde{\bfy_i}\}_{i=1}^{N_{\text{samples}}} \gets \{\tilde{\bfy_i}\}_{i=1}^{N_{\text{samples}}} + \ \sim \mathcal{N}(0,\sigma)$ {\scriptsize(add noise)}\;
        $\{\tilde{\bfy_i}\}_{i=1}^{N_{\text{samples}}} = \text{clip}(\{\tilde{\bfy_i}\}_{i=1}^{N_{\text{samples}}}, \bfy_{\text{min}}, \bfy_{\text{max}})$ {\scriptsize(clip to $\bfy$ bounds)} \;
        $\sigma \gets K\sigma $ {\scriptsize (shrink sampling scale)} \;
    }
  }
 $\hat{\bfy} = \arg\max(\{\tilde{p}_i\}, \{\tilde{\bfy}_i\}$)
 \caption{Derivative-Free Optimizer}
\end{algorithm}
Where $\text{Multinomial}(N_{\text{samples}},\{\tilde{p}_i\}_{i=1}^{N_{\text{samples}}},\{\tilde{\bfy_i}\}_{i=1}^{N_{\text{samples}}})$ refers to sampling $N_{\text{samples}}$ times from the multinomial distribution with probabilities $\{\tilde{p}_i\}_{i=1}^{N_{\text{samples}}}$ returning associated elements $\{\tilde{\bfy_i}\}_{i=1}^{N_{\text{samples}}}$.  For simplicity the noise is written as being drawn from $\sim \mathcal{N}(0,\sigma)$, but this should be an $N_{\text{samples}}$-dimensional vector with an independent Gaussian noise sample for each element. This algorithm is very similar to the Cross Entropy Method \cite{de2005tutorial}, but has a few differences: (i) our algorithm does not use a fixed number of elites, (ii) re-sampling with replacement, and (iii) we shrink the sampling variance via a prescribed schedule rather than computing empirical variances.  We typically use $\sigma_{\text{init}}=0.33,  \ K=0.5, \ N_{\text{iters}}=3, \ N_{\text{samples}}=16,384$, unless otherwise noted.

While the above method works great for up to $\bfy$ of 5 dimensions or less (Sec.~\ref{subsec:comparison-of-variants}), we look at both autoregressive and gradient-based methods for scaling to higher dimensions.

\subsection{Method with Autoregressive Derivative-Free Optimization.}\label{subsec:dfo-autoregressive}

In the autoregressive version we interleave training and inference with $m$ models, for $\bfy \in \mathbb{R}^m$, i.e. one model $E_{\theta}^j(\bfx, \bfy^{:j})$ for each dimension $j = 1, 2, ..., m$.  Model $E_{\theta}^j(\bfx, \bfy^{:j})$ takes in all $\bfy$ dimensions up to $j$.  This isolates sampling to one degree of freedom at a time, and enables scaling to higher dimensional action spaces.  For more on autoregressive energy models, see \cite{nash2019autoregressive}.

\begin{algorithm}[H]
\SetAlgoLined
\KwResult{$\hat{\bfy}$}
 Initialize: $\{\tilde{\bfy_i}\}_{i=1}^{N_{\text{samples}}} \sim \mathcal{U}(\mathbf{y}_{\text{min}}, \mathbf{y}_{\text{max}})$, $\sigma = \sigma_{\text{init}}$  \;
  \For{iter in 1, 2, ..., $N_{\text{iters}}$}{
    \For{$j$ in 0, 1, ..., $m$}{
        $\{E_i\}_{i=1}^{N_{\text{samples}}} = \{E^j_{\theta}(\bfx, \tilde{\bfy}^{:j}_i)\}_{i}^{N_{\text{samples}}}$ \ {\scriptsize (compute energies)}\;
        $\{\tilde{p}_i\}_{i=1}^{N_{\text{samples}}} = \{\frac{e^{-E_i}}{\sum_{j=1}^{N_{\text{samples}}} e^{-E_j}} \} \ \ $ {\scriptsize (softmax)}\;
        $\rightarrow$ {\em{if training, apply $\mathcal{L}_{\text{InfoNCE}}$ and update parameters of $E_{\theta}^j$}}
        }
        \If{iter < $N_{\text{iters}}$}{
            $\{\tilde{\bfy}_i^{:j}\}_{i=1}^{N_{\text{samples}}} \gets  \ \sim \text{Multinomial}(N_{\text{samples}},\{\tilde{p}_i\}_{i=1}^{N_{\text{samples}}},\{\tilde{\bfy}^{:j}_i\}_{i=1}^{N_{\text{samples}}})$ {\scriptsize (resample with replacement)}\;
            $\{\tilde{\bfy}^{j}_i\}_{i=1}^{N_{\text{samples}}} \gets \{\tilde{\bfy}^{j}_i\}_{i=1}^{N_{\text{samples}}} + \ \sim \mathcal{N}(0,\sigma)$ {\scriptsize(add noise)}\;
            $\{\tilde{\bfy_i}\}_{i=1}^{N_{\text{samples}}} = \text{clip}(\{\tilde{\bfy_i}\}_{i=1}^{N_{\text{samples}}}, \bfy_{\text{min}}, \bfy_{\text{max}})$ {\scriptsize(clip to $\bfy$ bounds)} \;
            $\sigma \gets K\sigma $ {\scriptsize (shrink sampling scale)} \;
        }
  }
 $\hat{\bfy} = \arg\max(\{\tilde{p}_i\}, \{\tilde{\bfy}_i\}$)
 \caption{Autoregressive Derivative-Free Optimizer}
\end{algorithm}

\subsection{Method with Gradient-based, Langevin MCMC} \label{subsec:langevin}

For gradient-based MCMC (Markov Chain Monte Carlo) training we use the approach described in \cite{du2019implicit, mordatch2018concept} which uses stochastic gradient Langevin dynamics (SGLD) \cite{welling2011bayesian}:  
%
%
$${\color{red} ^k\tilde{\mathbf{y}}_i^j} =  {\color{red} ^{k-1}\tilde{\mathbf{y}}_i^j} - \lambda \big(\frac{1}{2}\nabla_{\mathbf{y}} E_{\theta}(\mathbf{x}_i, \ {\color{red} ^{k-1}\tilde{\mathbf{y}}^j_i}) + \omega^k \big) , \ \ \ \omega^k \sim \mathcal{N}(0, \sigma)$$
Note that in the conditional case, 
$\nabla$ is
respect to only $\mathbf{y}$, and not $\mathbf{x}$. 
As in \cite{du2019implicit, mordatch2018concept} we initialize ${\color{red}\{ ^0\tilde{\mathbf{y}} \}}$ from the uniform distribution, similar to Sec.~\ref{subsec:dfo}, but then optimize these contrastive samples with MCMC.  For each $N_{\text{neg}}$, we run $N_{\text{MCMC}}$ steps of the MCMC chain.  As recommended in \cite{grathwohl2019your} we use a polynomially-decaying schedule for the step-size $\lambda$.  
Note backpropagation is not performed backwards through the chain, but rather a \texttt{stop\_gradient()} is used after implicitly generating the samples \cite{du2019implicit}.  Also as in \cite{du2019implicit} we clip gradient steps, choosing to clip the full $\Delta \bfy$ value, i.e. after the gradient and noise have been combined.  Additionally for inference we run the Langevin MCMC chain a second time, giving twice as many inference Langevin steps as were used during training.  Also, for Langevin, all $\{\bfy\}$ (i.e. $\bfa$ for actions), in the training dataset are normalized per-dimension to span the range $[\bfy_{\text{min}}=-1, \  \bfy_{\text{max}}=1]$.

\subsubsection{Gradient Penalty}

For additional stability of training, we use both spectral normalization \cite{miyato2018spectral} as in \cite{du2019implicit}, and also add gradient penalties.  Gradient penalties are well known in the GAN community, and the form of our gradient penalty is inspired by \cite{gradientpenalty2021}:
$$\mathcal{L}_{grad} = \sum_{i=1}^N \sum_{j=1}^{N_{\text{neg}}} \sum_{k=\{\cdot\}} \max \bigg(0 \ , \ (||\nabla_{\mathbf{y}} E_{\theta}(\mathbf{x}_i, {\color{red} ^k\tilde{\mathbf{y}}^j_i})||_\infty - M) \bigg)^2 $$
Where the sums over $i$, $j$, $k$, represent respectively the sum over training samples, counter-examples per each data sample, and some subset of iterative chain samples for which we find it is sufficient to use only the final step, $k=\{N_{\text{MCMC}}\}$. 
%
$M$ controls the scale of the gradient relative to the noise $\omega$ in SGLD.  If $M$ is too large, then the noise in SGLD has little effect; if $M$ is too small, then the noise overpowers the gradient.  Empirically we find $M=1$ is a good setting. 
On each step of training, the gradient penalty loss is simply added to the InfoNCE loss, i.e. $\mathcal{L} = \mathcal{L}_{\text{grad}} + \mathcal{L}_{\text{InfoNCE}}$. Lastly, we note there are other approaches for improving stability of Langevin-based training, such as loss functions with entropy regularization \cite{du2020improved}.


To aid intuition on why constraints on the gradients $\nabla E(\cdot)$ are allowable restrictions for the model, Corollary 1.1 shows that the energy model is capable of having an arbitrary Lipschitz constant.

\subsection{Comparison of EBM Variants}\label{subsec:comparison-of-variants}

\begin{wrapfigure}{r}{0.43\textwidth}
  \vspace{-1.8em}
  \begin{center}
    \includegraphics[width=0.4\textwidth]{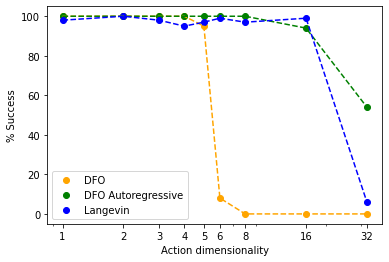}
  \end{center}
  \vspace{-1em}
  \caption{Comparison of used EBM methods on the $N$-D particle environment, showing methods using DFO (derivative-free optimization, Sec.~\ref{subsec:dfo}), autoregressive DFO (Sec.~\ref{subsec:dfo-autoregressive}), or Langevin dynamics (Sec.~\ref{subsec:langevin}).}
  \label{fig:ebm-variants}
  \vspace{-1.0em}
\end{wrapfigure}

A key comparison between these methods is the tradeoff of simplicity for higher-dimensional action spaces. As shown in Fig.~\ref{fig:ebm-variants}, with only 2,000 demonstrations in the $N$-D particle environment, the joint-dimensions-optimized derivative-free version (Sec.~\ref{subsec:dfo}) fails to solve the environment past $N=$5 dimensions, due to the curse of dimensionality and its naive sampling.  Both the autoregressive (Sec.~\ref{subsec:dfo-autoregressive}) and Langevin (Sec.~\ref{subsec:langevin}) versions are able to solve the environment reliably up to 16 dimensions, and with nonzero success at 32 dimensions.  The autoregressive version requires no new gradient stabilization, and can use only the same loss function, $\mathcal{L}_{\text{InfoNCE}}$, but is memory-intensive, requiring $N$ separate models for $N$ dimensions.  The Langevin version scales to high dimensions with only one model, but requires gradient stabilization.  For more on autoregressive and Langevin generative EBMs, see \cite{nash2019autoregressive} and \cite{du2019implicit,du2020improved}. Which variant is used for each of our evaluation tasks is enumerated in Section \ref{policy_learning_results_overview_and_protocol}.

\section{Additional Experimental Details and Analysis}

\subsection{Per-Task Summary of \# Demonstrations and Environment Dimensionalities}

In this section, with the table below, we highlight key aspects of the different evaluated policy learning experimental tasks, specifically the \# of demonstrations for each task and the dimensionalities of the environments (comprised of the observation spaces, state spaces, and action spaces).  As is highlighted in the table, the various tasks cover a wide set of challenges, including: low-data-regime tasks, and tasks with high observation, state, and/or action dimensionalities.

\begin{table}[H]
  \centering
  \tiny
  \setlength\tabcolsep{5.3pt}
  \begin{tabular}{@{}llcccccc@{}}
  \toprule
    & & \multicolumn{1}{c}{Demonstrations} & \multicolumn{3}{c}{Dimensionalities} \\
    \cmidrule(lr){3-3} \cmidrule(lr){4-6}
   \emph{Domain} & \emph{Task Name}  & \#  & {\textit{Observations}} & {\textit{States}} & {\textit{Actions}}& Results Shown In & Comment  \\
  \hline
  \hline
  \multirow{7}{*}{D4RL Human-Experts} & kitchen-complete & {\textbf{\textcolor{red}{19}}} & {\textbf{\textcolor{OliveGreen}{60}}} & {\textbf{\textcolor{blue}{60}}} & {\textbf{\textcolor{cyan}{9}}} & \multirow{7}{*}{Table 2} \\
                          & kitchen-partial & 601 & {\textbf{\textcolor{OliveGreen}{60}}} & {\textbf{\textcolor{blue}{60}}} & {\textbf{\textcolor{cyan}{9}}} \\
                          & kitchen-mixed & 601 & {\textbf{\textcolor{OliveGreen}{60}}} & {\textbf{\textcolor{blue}{60}}} & {\textbf{\textcolor{cyan}{9}}} \\
                          & pen-human & {\textbf{\textcolor{red}{50}}} & {\textbf{\textcolor{OliveGreen}{45}}} & {\textbf{\textcolor{blue}{45}}} & {\textbf{\textcolor{cyan}{24}}} \\
                          & hammer-human & {\textbf{\textcolor{red}{25}}} & {\textbf{\textcolor{OliveGreen}{46}}} & {\textbf{\textcolor{blue}{46}}} & {\textbf{\textcolor{cyan}{26}}} \\
                          & door-human & {\textbf{\textcolor{red}{25}}} & {\textbf{\textcolor{OliveGreen}{39}}} & {\textbf{\textcolor{blue}{39}}} & {\textbf{\textcolor{cyan}{28}}} \\
                          & relocate-human & {\textbf{\textcolor{red}{25}}} & {\textbf{\textcolor{OliveGreen}{39}}} & {\textbf{\textcolor{blue}{39}}} & {\textbf{\textcolor{cyan}{30}}} \\
  \hline
  \hline
  \multirow{9}{*}{Particle Integrator} & "1D"-Particle & 2,000 & 4 & 4 & 1 & \multirow{9}{*}{Figure 6} \\
                                       & "2D"-Particle & 2,000 & 8 & 8 & 2  \\
                                       & "3D"-Particle & 2,000 & 12 & 12 & 3  \\
                                       & "4D"-Particle & 2,000  & 16 & 16 & 4  \\
                                       & "5D"-Particle & 2,000  & 20 & 20 & 5  \\
                                       & "6D"-Particle & 2,000  & 24 & 24 & 6  \\
                                       & "8D"-Particle & 2,000  & {\textbf{\textcolor{OliveGreen}{32}}} & {\textbf{\textcolor{blue}{32}}} & 8  \\
                                       & "16D"-Particle & 2,000 & {\textbf{\textcolor{OliveGreen}{64}}} & {\textbf{\textcolor{blue}{64}}} & {\textbf{\textcolor{cyan}{16}}}  \\
                                       & "32D"-Particle & 2,000 & {\textbf{\textcolor{OliveGreen}{128}}} & {\textbf{\textcolor{blue}{128}}} & {\textbf{\textcolor{cyan}{32}}}  \\
  \hline
  \hline
  \multirow{3}{*}{Simulated Pushing} & Single Target, States & 2,000 & 10 & 10 & 2 & \multirow{3}{*}{Table 3}  \\
                                     & Multi Target, States & 2,000 & 13 & 13 & 2 \\
                                     & Single Target, Pixels & 2,000 & {\textbf{\textcolor{OliveGreen}{129,600}}} & 10 & 2 & & 180x240x3 image\\
  \hline
  \hline
  \multirow{2}{*}{Planar Sweeping} & Image input  & {\textbf{\textcolor{red}{50}}} & {\textbf{\textcolor{OliveGreen}{27,648}}} & {\textbf{\textcolor{blue}{203}}} & 3 & \multirow{2}{*}{Table 4} & 96x96x3 image\\
                                   & State input & {\textbf{\textcolor{red}{50}}} & {\textbf{\textcolor{OliveGreen}{203}}} & {\textbf{\textcolor{blue}{203}}} & 3  & & \\
  \hline
  \hline
  \multirow{1}{*}{Bi-Manual Sweeping} & Image-and-state input & 1,000 & {\textbf{\textcolor{OliveGreen}{27,660}}} & {\textbf{\textcolor{blue}{372}}} & {\textbf{\textcolor{cyan}{12}}} & \multirow{1}{*}{Table 5} & 96x96x3 image  \\
  \hline
  \hline
  \multirow{4}{*}{Real Robot} & Push-Red-Then-Green & {\textbf{\textcolor{red}{95}}} & {\textbf{\textcolor{OliveGreen}{32,400}}} & 8 & 2 & \multirow{4}{*}{Table 6} & \multirow{4}{*}{90x120x3 image.} \\
                              & Push-Red/Green-Multimodal & 410 & {\textbf{\textcolor{OliveGreen}{32,400}}} & 8 & 2 \\
                              & Insert-Blue & 223 & {\textbf{\textcolor{OliveGreen}{32,400}}} & 8 & 2 \\
                              & Sort-Blue-From-Yellow & 502 & {\textbf{\textcolor{OliveGreen}{32,400}}} & {\textbf{\textcolor{blue}{26}}} & 2 \\
  
  \bottomrule
  \end{tabular}
  \caption{\scriptsize Summary of the \# demonstrations and {\textit{observation/state/action-}}dimensionalities for each of the environments used in policy learning experiments.  Highlighted in color are {\textbf{\textcolor{red}{(red), low-data-regime tasks}}} with \# demos under 100, {\textbf{\textcolor{OliveGreen}{(green), high observation dimensionality}}} above 25, {\textbf{\textcolor{blue}{(blue), high state dimensionality}}} above 25, and  {\textbf{\textcolor{cyan}{(cyan), high action dimensionality}}} at or above 9.}
  \label{table:summary-of-demo-dims-table}
  \vspace{-2.5em}
\end{table}

\subsection{Training and Inference Times, Implicit vs. Explicit Comparison}\label{sec:timing}

\textbf{D4RL Train+Eval Times.} Table \ref{table:d4rl-train-eval-times} compares example training + evaluation times for the chosen best-performing models on the D4RL tasks.  We report both the training steps/second, and then also the full time for running an experiment, which comprises training to 100k steps with intermittently evaluating 100 episodes every 10k steps.
\begin{table}[H]
\centering
\scriptsize
\begin{tabular}{|l|l|l|l|}
\hline
                                & \textbf{Implicit BC} & \textbf{Explicit BC} & Comment \\
\hline
Configuration                   & As in Section \ref{subsec:d4rl-configs} & As in Section \ref{subsec:d4rl-configs} &\\
{\em{Summary:}}                          & 512 batch size                 & 512 batch size &\\
                                         & 512x8 MLP                      & 2048x8 MLP &\\
                                         & 100 Langevin iterations        & &\\
                                         & 8 counter examples        & &\\
\hline
Device                                   & TPUv3                 & TPUv3 &\\
\hline
Task                                     & door-human-v0         & door-human-v0 &\\
\hline
Training rate (steps/sec)                & 17.9             & 101.3 &  \\
Total train + eval time (hrs)              & 3.4              & 0.66  &  100k train steps, 100 evals every 10k steps\\
\hline
\end{tabular}
\caption{Comparison of training+evaluation times for implicit vs. explicit models on an example D4RL task.}
\label{table:d4rl-train-eval-times}
\end{table}
As is shown in Table \ref{table:d4rl-train-eval-times}, the best-performing implicit models, which are 100-iteration Langevin models, take 5.6x the train+eval time compared to the best-performing explicit models.  Note that even the 3.4-hour full train+eval time for the implicit model is considerably faster than what has been reported \cite{kostrikov2021offline} for completing a train+eval on a comparable D4RL task for CQL: 16.3 hours.  

\textbf{Real-World Image-based Train and Inference Times.} The following compares relevant training and inference times for our real-world tasks.  In contrast to the D4RL scenario discussed above, in this scenario (a) there are large image observations to process, and (b) there are no simulated evaluations run during training.  We report the training steps/sec  rate, as well as the total train time, which is performed on a server of 8 GPUs.  Once trained, the model is then deployed on a single-GPU machine, for which we report the inference times.
\begin{table}[H]
\centering
\scriptsize
\begin{tabular}{|l|l|l|l|}
\hline
                                & \textbf{Implicit BC} & \textbf{Explicit BC} & Comment \\
\hline
Configuration                   & As in Section \ref{subsec:real-robot-configs}              & As in Section \ref{subsec:real-robot-configs} &\\
{\em{Summary:}}                          & 128 batch size                 & 128 batch size &\\
                                         & 90x120 images                  & 90x120 images  &\\
                                         & 4-layer ConvMaxPool            & 4-layer ConvMaxPool &\\
                                         & 1024x4 MLP                     & 1024x4 MLP &\\
                                         & 256 counter examples        & &\\
\hline
Training Device                                   & 8x V100 GPU                 & 8x V100 GPU &\\
\hline
Task                                     & Push-Red-Then-Green         & Push-Red-Then-Green &\\
\hline
Training rate (steps/sec)                & 4.7             & 5.5 &  \\
Total train time (hrs)                   & 5.0             & 5.8 &  100k train steps \\
\hline
Inference Device                         & 1x RTX 2080 Ti GPU & 1x RTX 2080 Ti GPU & \\
\hline
Inference parameters           & 1024 samples & & \\
                               & 3 dfo iterations & & \\
\hline
Inference time (ms)                      & 7.22 & 3.49  &  \\
\hline
\end{tabular}
\caption{Example comparison of training  and inference times for implicit vs. explicit models used for a Real Robot task.}
\label{table:real-train-inference-times}
\end{table}
Table~\ref{table:real-train-inference-times} shows that for these visual models, the training times are reasonably comparable for the implicit and explicit models -- 5.0 and 5.8 hours respectively.  Compared to the previous D4RL scenario, this can be explained because the training time is mostly dominated by visual processing.  As the implicit models use late fusion (Sec.~\ref{subsec:architectures}), the visual processing time is identical to the explicit models.  For inference, the chosen implicit models show a modest increase in inference time, up to 7.22 milliseconds (ms) from 3.49 ms for the explicit model.  This can be attributed to time spent on the iterative derivative-free optimization. Note that the inference time of the implicit model can be adjusted by adding/decreasing the number of samples and iterations.  For example, using the same trained model but increasing the samples from 1024 to 2048 causes the inference time to increase to 9.25 ms.

\subsection{Additional Real-World Experimental Details}\label{sec:appendix-real}

\subsubsection{Robot Hardware Configuration, Workspace, and Objects}

Our real-world experiments make use of a UFACTORY xArm6 robot arm with all state logged at 100 Hz. Observations are recorded from an Intel RealSense D415 camera, using RGB-only images at 640x360 resolution, logged at 30 Hz. The cylindrical end-effector is made from a 6 inch long plastic PVC pipe sourced from McMaster-Carr (\href{https://www.mcmaster.com/9173K515/}{9173K515}). The work surface is 24 x 18 inch smooth wood cutting board. The manipulated objects are from the Play22 Baby Blocks Shape Sorter toy kit (\href{https://play22usa.com/shop/ols/products/16olfxvr5t}{Play22}). The targets for the tasks were constructed by hand out of wood and spraypainted black.  All demonstrations were provided by a mouse-based interface for providing real-time demonstrations.

The 6DOF robot is constrained to move in a 2D plane above the table.  This aids in safety of the robot during operation, since it is constrained to not collide with the table and cannot provide normal forces against objects down into the table either.

\subsubsection{Robot Policy and Controller}

The learned visual-feedback policy operates at 5 Hz.  On a GTX 2080 Ti GPU, the implicit models (configuration in Sec.~\ref{policy_learning_results_overview_and_protocol}.5) complete inference in under 10 ms (see Sec.~\ref{sec:timing}), and so could be run faster than 5 Hz, but we find 5 Hz to be sufficient.  The learned action space is a delta Cartesian setpoint, from the previous setpoint to the new one.   The setpoints are linearly interpolated from their 5 Hz rate to be 100 Hz setpoints to our joint level controller.  The joint level controller uses PyBullet \cite{coumans2016pybullet} for inverse kinematics, and sends joint positions to the xArm6 robot at 100 Hz.














\subsection{Nearest-Neighbor Baseline}

This baseline memorizes all training data, and performs inference by looking up the closest observation in the training set and returning the corresponding action.  Specifically, given a finite training dataset of pairs $\{(\bfx, \bfy)\}_i$, denote the inputs as $X = \{\bfx\}_i$ and outputs $Y = \{\bfy\}_i$, preserving the ordering in both $X$ and $Y$.  Given some new observation $\bfx'$, the Nearest-Neighbor model, $N(\cdot)$, computes:
$$N(\bfx')=Y[\underset{i}{\arg\min} \ |\bfx' - X[i]| \ ]$$
for some norm $|\cdot|$.  Specifically we used L2 norm.  We experimented with normalizing all observations per-dimension to be unit-variance, but did not find this to improve results.  For environments with state-only observations (no images), we can compute this exactly and quickly all in processor memory, but for the image-observation Simulated Pushing task we tested, the dataset did not fit in memory.  Accordingly, we used a random linear projection, which is known to be a viable method for nearest-neighbor lookup of image data \cite{bingham2001random}, from the observation space to a 128-dimensional vector.  We then stored all these 128-dimensional vectors in memory, and used these for Nearest-Neighbor lookups.




\subsection{$N$-D Particle Environment Description}\label{nd_particle_environmnet_description}

In this environment, the agent (\ie particle) moves from its current configuration $q\in\mathbb{R}^N$ to a goal configuration $g_{0}\in\mathbb{R}^N$, followed by a second goal configuration $g_{1}\in\mathbb{R}^N$. Given its position $q$ and velocity $\dot{q}$, its action is the target position $\hat{q}\in\mathbb{R}^N$ applied to a PD controller which computes acceleration $\ddot{q}$ according to: $\ddot{q}=k_p(\hat{q} - q)+k_d(\hat{\dot{q}}-\dot{q})$ where target velocity $\hat{\dot{q}} = 0$, and $k_p$ and $k_d$ are environment-fixed constant gains. Initial and goal particle configurations are randomized, for each dimension, in the range $[0,1]$ for each episode, and differ between training and testing. To generate demonstrations, a scripted policy returns actions $q=g_{0}$ until the agent falls within a radius $r$ of $g_{0}$, then returns actions $q=g_{1}$ until the agent falls within a radius $r$ of $g_{1}$. Agent state and goal positions are used as input to the policy, which is trained to imitate the behavior of the scripted policy and tested on its capacity to generalize to new goal configurations. This task can be thought of as modeling an $N$-dimensional step function while dealing with compounding errors. The mode switch between goals presents a discontinuity that needs to be learned.

\subsection{Analysis: Training Data Sparsity in the $N$-D Particle Tasks}

\begin{wrapfigure}{r}{0.43\textwidth}
  \vspace{-1.8em}
  \begin{center}
    \includegraphics[width=0.4\textwidth]{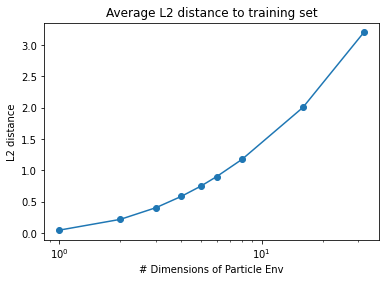}
  \end{center}
  \vspace{-1em}
  \caption{Depiction of {\em{training data sparsity}} on the $N$-D particle tasks, as $N$ is varied.  Shown, for each $N$-D variant of the task, is the average distance of an evaluation episode initialization to the training set of 2,000 demonstrations.}
  \label{fig:particle-sparsity}
  \vspace{-1.0em}
\end{wrapfigure}

To complement other analyses on generalization, sample complexity, and interpolation/extrapolation, we analyze in Fig.~\ref{fig:particle-sparsity} another notion of generalization: training data sparsity.  In the $N$-D particle experiments, as we increase N but hold the number of demonstrations constant, the training data effectively becomes much sparser over the observation space.  New test-time environments for evaluation are accordingly, as $N$ increases, on average farther and farther away from the training set.  This helps explain how the Nearest-Neighbor baseline cannot solve this task well past 1D, since memorizing the training data is insufficient, and to succeed in a higher-dimensional environment a model must generalize.  This analysis complements our simple 1D->1D figures on extrapolation/interpolation (Fig.~2 and Fig.~3 in the main paper) and our visual generalization and sample complexity analysis (Fig.~4 in the main paper).

\subsection{Additional D4RL tasks}

In the main paper we focused on the human-expert tasks from D4RL, but here provide results on additional D4RL tasks as well.  Note that the other tasks shown, except for `random', use a reinforcement-learning-trained agent for the task, and this reinforcement-learning agent itself has a policy that is a uni-modal continuous, explicit function approximator, and it was optimized as such. Additionally, as expected, supervised imitation learning methods, which do not make use of the additional reward information from the provided demonstrations, perform comparatively worse on tasks with sub-optimal demonstrations. This is true of all tasks with ``*medium*'' and ``*random'' in their task name. Additionally, as stated in Section \ref{policy_learning_results_overview_and_protocol}, we choose the EBM hyperparameters to maximize performance on the human-expert based environments (``Franka'' and ``Adroit'' tasks) at the expense of lower performance on the ``Gym''-mujoco tasks. However, for fair comparison with other methods, and according to the standard D4RL evaluation protocol, a single set of hyperparameters was used for all tasks rather than presenting results that maximize each environment.


\begin{table}[H]
  \centering
  \tiny
  \begin{tabular}{@{}llccccccccc@{}}
  \toprule
    & & \multicolumn{3}{c}{Baselines} & \multicolumn{4}{c}{Ours} \\
    \cmidrule(lr){3-5} \cmidrule(lr){6-9}
         &  &              &     &      & {\em{Explicit}} & {\em{Implicit}} & {\em{Explicit}} & {\em{Implicit}}  \\
  Method &  & BC           & CQL \cite{kumar2020conservative} & S4RL \cite{sinha2021s4rl} & BC (MSE) & BC (EBM) & BC (MSE) & BC (EBM)  \\
         &  & (from CQL)   &     &      &        &        & w/ RWR \cite{peters2007reinforcement} & w/ RWR \cite{peters2007reinforcement} \\
  \midrule
  Uses data & & $(\mathbf{o}, \mathbf{a})$ & $(\mathbf{o}, \mathbf{a}, r)$ & $(\mathbf{o}, \mathbf{a}, r)$ & $(\mathbf{o}, \mathbf{a})$ & $(\mathbf{o}, \mathbf{a})$ & $(\mathbf{o}, \mathbf{a}, r)$ & $(\mathbf{o}, \mathbf{a}, r)$\\
  \midrule
  \emph{Domain} & \emph{Task Name} \\
  \hline
  \hline
  \multirow{3}{*}{Franka} & kitchen-complete & 1.4 & 1.8 & 3.08 & 1.76 $\pm$0.04 & \textbf{3.37} $\pm$0.19 & 1.22 $\pm$0.18 & \textbf{3.37} $\pm$0.01 \\  
                          & kitchen-partial  & 1.4 & 1.9 & \textbf{2.99} & 1.69 $\pm$0.02 &  1.45 $\pm$0.35 & 1.86 $\pm$0.26 & 2.18 $\pm$0.05 \\
                          & kitchen-mixed    & 1.9 & 2.0 &               & \textbf{2.15} $\pm$\textbf{0.06} &  1.51 $\pm$0.39 & 2.03 $\pm$0.06 & \textbf{2.25} $\pm$\textbf{0.14} \\
  \hline
  \hline
  \multirow{5}{*}{Adroit} & pen-human            & 1121.9 & 1214.0 & 1419.6 & 2141 $\pm$109 & \textbf{2586} $\pm$\textbf{65} & 2108 $\pm$58.8 & \textbf{2446} $\pm$\textbf{207}\\
                          & hammer-human         & -82.4  & 300.2  &  \textbf{496.2} & -38 $\pm$25 & -133 $\pm$26  & -35.1 $\pm$45.1 & -9.3  $\pm$45.5\\
                          & door-human           & -41.7  & 234.3  &  \textbf{736.5} & 79 $\pm$15 & 361 $\pm$67  & 17.9 $\pm$ 13.8 & 399 $\pm$34\\
                          & relocate-human       & -5.6   & 2.0    &    2.1 & -3.5 $\pm$1.1 & -0.1 $\pm$2.4 & -3.7 $\pm$0.3 & $\textbf{3.6} \pm\textbf{2.5}$\\
  \hline
  \hline
  \multirow{16}{*}{Gym}   & halfcheetah-medium          & 4202  & 5232  & 5778 & 4273  & 4086  & & \\
                          & walker2d-medium             & 304   & 3637  & 4298 & 822   & 676   & & \\
                          & hopper-medium               & 923   & 1867  & 2548 & 966   & 2430  & & \\
                          & halfcheetah-medium-replay   &       & 4934  & 6101 & 4029  & 2766  & & \\
                          & walker2d-medium-replay      &       & 970   & 1392 & 480   & 433   & & \\
                          & hopper-medium-replay        &       & 940   & 1132 & 543   & 382   & & \\
                          & halfcheetah-medium-expert   & 4164  & 7467  & 9528 & 11758 & 4040  & & \\
                          & walker2d-medium-expert      & 520   & 4533  & 5152 & 640   & 745   & & \\
                          & hopper-medium-expert        & 3621  & 3592  & 3674 & 909   & 876   & & \\
                          & halfcheetah-expert          & 13004 & 12731 &      & 12802 & 9436  & & \\
                          & walker2d-expert             & 5772  & 7067  &      & 2677  & 3746  & & \\
                          & hopper-expert               & 3527  & 3557  &      & 3619  & 3549  & & \\
                          & halfcheetah-random          & -118  & 4115  & 6213 & 0     & -392  & & \\
                          & walker2d-random             & 33    & 323   & 1145 & 145   & -1.63 & & \\
                          & hopper-random               & 308   & 331   & 331  & 284   & 308   & & \\
  \bottomrule
  \end{tabular}
  \caption{\scriptsize Baseline comparisons on D4RL \cite{fu2020d4rl} tasks, including Mujoco gym tasks.  Results shown are the average of 3 random training initialization seeds, 100 evaluations each.}
  \label{table:d4rl-table}
  \vspace{-1.5em}
\end{table}

\section{Policy Learning Results Overview and Protocol}\label{policy_learning_results_overview_and_protocol}

In each section below we describe the protocols for the individual simulation experiments.  Note that Figure 5 was produced by averaging the performance of the best policies, for each type, within each domain across the different tasks of that domain.

For EBM variants that were used for which task: Simulated Pushing and Real World, with action dimensionality of 2, used derivative-free optimization (Sec.~\ref{subsec:dfo}).  For Planar Sweeping, with action dimensionality 3, and Bi-Manual Sweeping, with action dimensionality 12, we used autoregressive derivative-free optimization (Sec.~\ref{subsec:dfo-autoregressive}).  D4RL, with action dimensionality between 3 and 30, used Langevin dynamics (Sec.~\ref{subsec:langevin}).  Particle, with action dimensionality between 1 and 32, used Langevin dynamics as well.  See Sec.~\ref{subsec:comparison-of-variants} for a comparison of variants.

\subsection{D4RL Experiments}\label{subsec:d4rl-configs}

For {\textbf{D4RL}} experiments, we run sweeps over several hyperparameters for both the Implicit BC (EBM) and Explicit MSE-BC models. We choose the final hyperparameters based on max average performance over 3 D4RL environments: hammer-human-v0, door-human-v0, and relocate-human-v0. We use the same final hyperparameters across all D4RL tasks for the final results. Note that we paid closest attention to the human-teleoperation task performance when selecting a single set of hyper parameters for D4RL, particularly at the expense of slightly lower task performance on the gym-mujoco D4RL tasks. For all evaluations, we report average results over 100 episodes for 3 seeds. To calculate the aggregate D4RL performance metric ``D4RL Human-Experts" in Figure 5 of the paper, we first calculated the normalized performance metric for the kitchen-complete, kitchen-partial, kitchen-mixed, pen-human, hammer-human, door-human and relocate-human environments, then calculated the average across all these tasks.

The following hyperparameters were used for D4RL evaluation:

\subsubsection*{D4RL Implicit BC (EBM)}

\begin{table}[H]
\scriptsize
\begin{tabular}{|l|l|l|}
\hline
\textbf{Hyperparameter}              & \textbf{Chosen Value} & \textbf{Swept Values}      \\
\hline
EBM variant                          & Langevin              & \\
train iterations                     & 100,000               &                            \\
batch size                           & 512                   &                            \\
learning rate                        & 0.0005                &                            \\
learning rate decay                  & 0.99                  &                            \\
learning rate decay steps            & 100                   &                            \\
network size (width x depth)         & 512x8                 & 128x32, 512x8              \\
activation                           & ReLU                  & swish, ReLU                \\
dense layer type                     & spectral norm         & regular, spectral norm     \\
train counter examples               & 8                     & 1, 8, 64                   \\
action boundary buffer               & 0.05                  & 0.001, 0.05                \\
gradient penalty                     & final step only       & all steps, final step only \\
gradient margin                      & 1                     & 0.6, 1.0, 1.3              \\
langevin iterations                  & 100                   & 100, 150                   \\
langevin learning rate init.         & 0.5                   & 2.0, 1.0, 0.5, 0.1         \\
langevin learning rate final         & 1.00E-05              & 1e-4, 1e-5, 1e-6           \\
langevin polynomial decay power      & 2                     & 2.0, 1.0                   \\
langevin delta action clip           & 0.5                  & 0.05, 0.1, 0.5             \\
langevin noise scale                 & 0.5                   & 0.5, 1.0                   \\
langevin 2nd iteration learning rate & 1.00E-05              & 1e-1, 1e-2, 1e-5          \\
\hline
\end{tabular}
\end{table}

Shown also is an indication of training stability, across 5 different seeds, shown for the pen task.

\begin{figure}[H]
\begin{subfigure}{.5\textwidth}
  \centering
  \includegraphics[width=.9\linewidth]{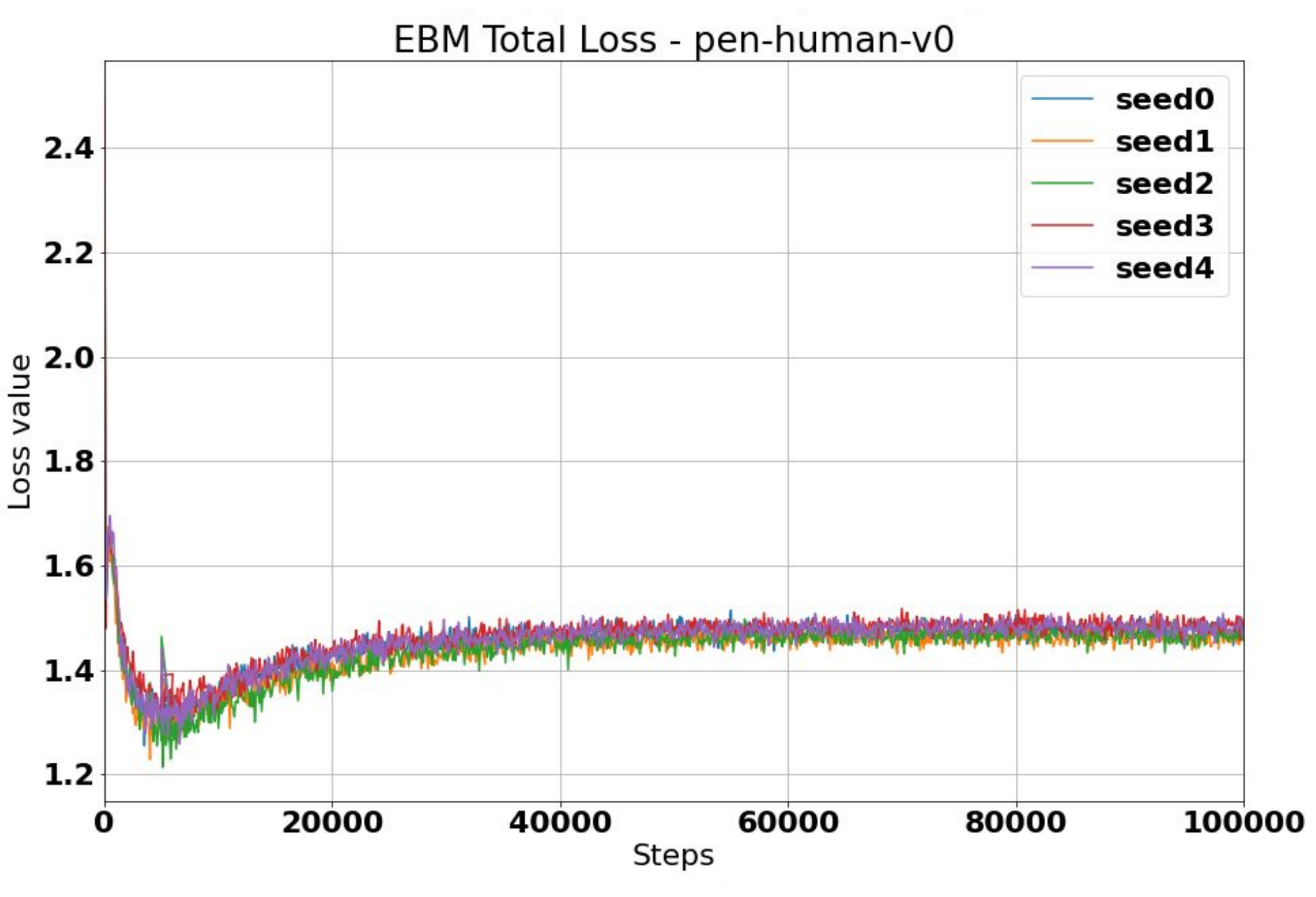}
  \caption{A plot of the total EBM loss on the pen-human-v0 D4RL environment for each of 5 seeds.  Note that with Langevin sampling, as the sample quality improves, the EBM loss can rise.}
  \label{fig:ebm-total-loss}
\end{subfigure}
\begin{subfigure}{.5\textwidth}
  \centering
  \includegraphics[width=.9\linewidth]{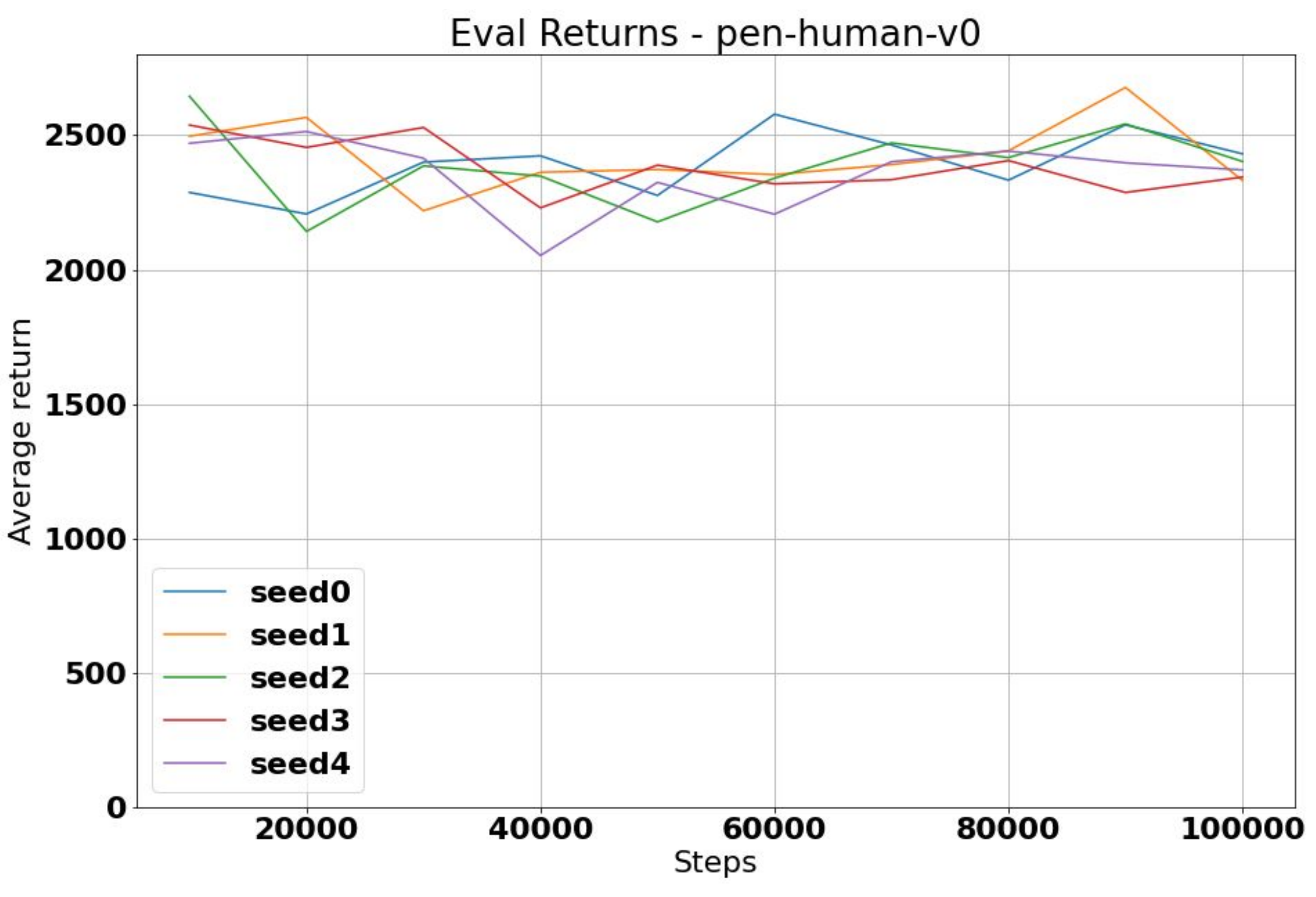}
  \caption{A plot of the eval returns from the same run on pen-human-v0 for 5 seeds, average of 100 evals.}
  \label{fig:eval-returns}
\end{subfigure}
\end{figure}

\subsubsection*{D4RL Explicit MSE-BC}
\begin{table}[H]
\scriptsize
\begin{tabular}{|l|l|l|}
\hline
\textbf{Hyperparameter}   & \textbf{Chosen Value} & \textbf{Swept Values}                                          \\
\hline
train iterations          & 100,000               &                                                                \\
batch size                & 512                   &                                                                \\
sequence length           & 2                     &                                                                \\
learning rate             & 0.001                 & 1e-3, 0.5e-3                                                   \\
learning rate decay       & 0.99                  &                                                                \\
learning rate decay steps & 200                   &                                                                \\
dropout rate              & 0.1                   &  0.0, 0.1                                                              \\
network size (width x depth)               & 2048x8                & 128x16, 128x32, 512x16, 512x32, 1024x4, 1024x8, 2048x4, 2048x8 
\\
activation                             & ReLU                  &                            \\
\hline
\end{tabular}
\end{table}

\subsection{Simulated Pushing Experiments}

For {\textbf{Simulated Pushing}} experiments, we run separate sweeps for each model for each of the States and Pixels versions of the task. All chosen hyperparameter sweeps and chosen values are given in tables below, and results are reported as the average of 100 episodes for 3 seeds.  

\subsubsection*{Simulated Pushing, States, Implicit BC (EBM)}

\begin{table}[H]
\scriptsize
\begin{tabular}{|l|l|l|}
\hline
\textbf{Hyperparameter}   & \textbf{Chosen Value} & \textbf{Swept Values}         \\
\hline
EBM variant               & DFO                   & \\
train iterations          & 100,000               &                               \\
batch size                & 512                   &                               \\
sequence length           & 2                     & 2, 4                          \\
learning rate             & 0.001                 &                               \\
learning rate decay       & 0.99                  &                               \\
learning rate decay steps & 100                   &                               \\
network size (width x depth)              & 128x8                 & 2048x4, 128x8, 128x16, 128x32 \\
activation                & ReLU                  &                               \\
dense layer type          & regular               &                               \\
train counter examples    & 256                   &                               \\
action boundary buffer    & 0.05                  &                               \\
gradient penalty          & none                  &                               \\
dfo samples               & 16384                 &                               \\
dfo iterations            & 3                     &  \\
\hline
\end{tabular}
\end{table}

\subsubsection*{Simulated Pushing, States, Explicit MSE-BC}
\begin{table}[H]
\scriptsize
\begin{tabular}{|l|l|l|}
\hline
\textbf{Hyperparameter}   & \textbf{Chosen Value} & \textbf{Swept Values}            \\
\hline
train iterations          & 100,000               &                                  \\
batch size                & 512                   &                                  \\
sequence length           & 2                     &                                  \\
learning rate             & 0.0005                & 4e-3, 2e-3, 1e-3, 0.5e-3, 0.2e-3 \\
learning rate decay       & 0.99                  &                                  \\
learning rate decay steps & 100                   & 100, 150, 200, 400               \\
dropout rate              & 0.1                   &                                  \\
network size (width x depth)               & 1024x8                & 1024x4, 1024x8, 2048x4, 2048x8  \\
activation                             & ReLU                  &                            \\
\hline
\end{tabular}
\end{table}

\subsubsection*{Simulated Pushing, States, Explicit MDN-BC}
\begin{table}[H]
\scriptsize
\begin{tabular}{|l|l|l|}
\hline
\textbf{Hyperparameter}   & \textbf{Chosen Value} & \textbf{Swept Values} \\
\hline
train iterations          & 100,000               &                       \\
batch size                & 512                   &                       \\
sequence length           & 2                     &                       \\
learning rate             & 0.001                 &                       \\
learning rate decay       & 0.99                  &                       \\
learning rate decay steps & 100                   &                       \\
dropout rate              & 0.1                   &                       \\
network size (width x depth)              & 512x8                 & 512x8, 512x16         \\
training temperature      & 1.0                   & 0.5, 1.0, 2.0         \\
test temperature          & 1.0                   & 0.5, 1.0, 2.0         \\
test variance exponent    & 1.0                   & 1.0, 4.0              \\
\hline
\end{tabular}
\end{table}

\subsubsection*{Simulated Pushing, Pixels, Implicit BC (EBM)}
\begin{table}[H]
\scriptsize
\begin{tabular}{|l|l|l|}
\hline
\textbf{Hyperparameter}   & \textbf{Chosen Value} & \textbf{Swept Values}                       \\
\hline
EBM variant               & DFO                   & \\
train iterations          & 100,000               &                                             \\
batch size                & 128                   & 128, 256                                    \\
sequence length           & 2                     &                                             \\
learning rate             & 0.001                 &                                             \\
learning rate decay       & 0.99                  &                                             \\
learning rate decay steps & 100                   &                                             \\
image size                & 240x180               & 120x90, 240x180                             \\
MLP network size (width x depth)               & 1024x4                & 512x4, 1024x4, 256x14, 256x26, 1024x14, 1024x26 \\
Conv. Net.                & 4-layer ConvMaxPool           & \\
activation                & ReLU                  &                                             \\
dense layer type          & regular               &                                             \\
train counter examples    & 256                   &                                             \\
action boundary buffer    & 0.05                  &                                             \\
gradient penalty          & none                  &                                             \\
dfo samples               & 4096                  & 1024, 4096, 16384                           \\
dfo iterations            & 3                     &                                            \\
\hline
\end{tabular}
\end{table}

\subsection*{Simulated Pushing Pixels MSE-BC}
\begin{table}[H]
\scriptsize
\begin{tabular}{|l|l|l|}
\hline
\textbf{Hyperparameter}   & \textbf{Chosen Value} & \textbf{Swept Values}      \\
\hline
train iterations          & 100,000               &                            \\
batch size                & 64                    &                            \\
sequence length           & 2                     &                            \\
learning rate             & 0.001                 &                            \\
learning rate decay       & 0.99                  &                            \\
learning rate decay steps & 100                   &                            \\
image size                & 240x180               & 120x90, 240x180            \\
dropout rate (MLP only)              & 0.1                   &                            \\
network size (width x depth)               & 512x4                 & 128x2, 128x4, 512x2, 512x4 \\
Conv. Net.                & 4-layer ConvMaxPool           & \\
activation                             & ReLU                  &                            \\
coord conv                & True                  & True, False               \\
\hline
\end{tabular}
\end{table}
\subsection*{Simulated Pushing Pixels MDN-BC}
\begin{table}[H]
\scriptsize
\begin{tabular}{|l|l|l|}
\hline
\textbf{Hyperparameter}   & \textbf{Chosen Value} & \textbf{Swept Values} \\
\hline
train iterations          & 100,000               &                       \\
batch size                & 32                    &                       \\
sequence length           & 2                     &                       \\
learning rate             & 0.001                 &                       \\
learning rate decay       & 0.99                  &                       \\
learning rate decay steps & 100                   &                       \\
dropout rate (MLP only)              & 0.1                   &                       \\
image size                & 120x90                & 120x90, 240x180       \\
network num components    & 26                    &                       \\
network size (width x depth)               & 512x8                 & 512x8, 512x16         \\
Conv. Net.                & 4-layer ConvMaxPool           & \\
activation                             & ReLU                  &                            \\
training temperature      & 2.0                   & 0.5, 1.0, 2.0         \\
test temperature          & 2.0                   & 0.5, 1.0, 2.0         \\
test variance exponent    & 4.0                   & 1.0, 4.0              \\
\hline
\end{tabular}
\end{table}

\subsection{Simulated $N$-D Particle Environment Experiments}

For a detailed description of this environment and its dynamics, see Section \ref{nd_particle_environmnet_description}. We used the following hyper parameters for evaluation on this environment:

\subsubsection*{Particle Implicit BC (EBM)}

\begin{table}[H]
\scriptsize
\begin{tabular}{|l|l|}
\hline
\textbf{Hyperparameter}              & \textbf{Chosen Value}  \\
\hline
EBM variant                          & Langevin               \\
train iterations                     & 50,000                  \\
batch size                           & 128                   \\
sequence length                      & 2                     \\
learning rate                        & 0.001                   \\
learning rate decay                  & 0.99                   \\
learning rate decay steps            & 100                   \\
network size (width x depth)         & 128x16                 \\
activation                           & ReLU                   \\
dense layer type                     & spectral norm          \\
train counter examples               & 64                   \\
gradient penalty                     & final step only       \\
gradient margin                      & 1                       \\
langevin iterations                  & 50                     \\
langevin learning rate init.         & 0.1                    \\
langevin learning rate final         & 1.00E-05              \\
langevin polynomial decay power      & 2                     \\
langevin delta action clip           & 0.1                   \\
langevin noise scale                 & 1.0                    \\
langevin 2nd iteration learning rate & not used              \\
\hline
\end{tabular}
\end{table}

\subsubsection*{Particle Explicit MSE-BC}

\begin{table}[H]
\scriptsize
\begin{tabular}{|l|l|}
\hline
\textbf{Hyperparameter}   & \textbf{Chosen Value}   \\
\hline
train iterations          & 100,000                                                                        \\
batch size                & 512                                                                               \\
sequence length           & 2                                                                                  \\
learning rate             & 0.001                                                                               \\
learning rate decay       & 0.99                                                                                 \\
learning rate decay steps & 200                                                                                 \\
dropout rate              & 0.1                                                                                 \\
network size (width x depth)              & 128x16               \\
activation                & ReLU                                                                        \\
\hline
\end{tabular}
\end{table}

\subsection{Simulated Sweeping Experiments}

For {\textbf{Planar Sweeping}}, for both explicit and implicit models, results are shown for different types of encoders, and different $\#$ of Dense ResNet layers (Sec.~\ref{subsec:architectures}) shown in the table, each is the average of 100 evaluations each across 3 different seeds.  The best models, for each implicit and explicit, were taken from Planar Sweeping and evaluated on {\textbf{Bi-Manual Sweeping}}. 

We used the following hyper parameters for evaluation on the simulated planar sweeping, and bi-manual sweeping environment:

\subsubsection*{Planar Sweeping Implicit BC (EBM)}
\begin{table}[H]
\scriptsize
\begin{tabular}{|l|l|l|}
\hline
\textbf{Hyperparameter}              & \textbf{Chosen Value} & \textbf{Swept Values}      \\
\hline
EBM variant                            & Autoregressive       &                            \\
train iterations                       & 1,000,000             &                            \\
batch size                             & 64                    &                            \\
sequence length                        & 2                     &                            \\
learning rate                          & 1e-4                  & 1e-3, 1e-4                 \\
Conv. Net.                & ConvResNet           & \\
\# encoder features                    & 64                    &                                                                \\
\# Conv ResNet encoder layers          & 26                    &                            \\
\# spatial softmax heads               & 64                    & 8, 16, 32, 64              \\
\# dense ResNet layers                 & 20                    & 8, 14, 20                  \\
activation                             & ReLU                  &                            \\
train counter examples per action dim  & 1024                  & 128, 256, 512, 1024        \\
inference examples per action dim      & 1024                  & 128, 256, 512, 1024        \\
\hline
\end{tabular}
\end{table}

\subsubsection*{Planar Sweeping Explicit MSE-BC}
\begin{table}[H]
\scriptsize
\begin{tabular}{|l|l|l|}
\hline
\textbf{Hyperparameter}   & \textbf{Chosen Value} & \textbf{Swept Values}                                          \\
\hline
train iterations          & 1,000,000             &                                                                \\
batch size                & 64                    &                                                                \\
sequence length           & 2                     &                                                                \\
learning rate             & 1e-4                  & 1e-3, 1e-4                                                     \\
Conv. Net.                & ConvResNet           & \\
\# encoder features       & 64                    &                                                                \\
\# Conv ResNet encoder layers          & 26                    &                            \\
\# spatial softmax heads  & 64                    & 8, 16, 32, 64                                                  \\
\# dense ResNet layers    & 20                    & 8, 14, 20                                                      \\
activation                & ReLU                  &                                                                \\
\hline
\end{tabular}
\end{table}

\subsubsection*{Bi-manual Sweeping Implicit BC (EBM)}
\begin{table}[H]
\scriptsize
\begin{tabular}{|l|l|}
\hline
\textbf{Hyperparameter}              & \textbf{Chosen Value}   \\
\hline
EBM variant                            & Autoregressive                             \\
train iterations                       & 1,000,000                                       \\
batch size                             & 32                                               \\
sequence length                        & 2                                               \\
learning rate                          & 1e-4                                             \\
Conv. Net.                & ConvResNet            \\
\# encoder features                    & 64                                             \\
\# Conv ResNet encoder layers          & 26                                              \\
\# spatial softmax heads               & 64                                                \\
\# dense ResNet layers                 & 20                                            \\
activation                             & ReLU                                            \\
train counter examples per action dim  & 1024                                            \\
inference examples per action dim      & 1024                                           \\
\hline
\end{tabular}
\end{table}

\subsubsection*{Bi-manual Sweeping Explicit MSE-BC}
\begin{table}[H]
\scriptsize
\begin{tabular}{|l|l|}
\hline
\textbf{Hyperparameter}   & \textbf{Chosen Value}                                        \\
\hline
train iterations          & 1,000,000                                                                         \\
batch size                & 32                                                                                  \\
sequence length           & 2                                                                                   \\
learning rate             & 1e-4                                                                                 \\
Conv. Net.                & ConvResNet            \\
\# encoder features       & 64                                                                                   \\
\# Conv ResNet encoder layers  & 26                                                                              \\
\# spatial softmax heads  & 64                                                                                   \\
\# dense ResNet layers    & 20                                                                                   \\
activation                & ReLU                                                                                  \\
\hline
\end{tabular}
\end{table}

\subsection{Real-world Pushing Experiments}\label{subsec:real-robot-configs}

For {\textbf{Real World}}, explicit and implicit models were taken from Simulated Pushing, Pixels, and applied to the real world.  We used the following hyper parameters for evaluation on the real-world pushing environments:

\subsubsection*{Real-world Tasks Pixels Implicit BC (EBM)}

\begin{table}[H]
\scriptsize
\begin{tabular}{|l|l|l|l|l|}
\hline
\textbf{Hyperparameter}   & \textbf{Pushing}      & \textbf{Pushing Multimodal} & \textbf{Insertion} & \textbf{Sorting} \\
\hline
EBM variant               & DFO                   & DFO                         & DFO                 & DFO             \\
train iterations          & 100,000               & 100,000                     & 100,000             & 100,000             \\
batch size                & 128                   & 256                         & 256                 & 256             \\
sequence length           & 2                     & 2                           & 2                   & 2             \\
learning rate             & 0.001                 & 0.001                       & 0.001               & 0.001             \\
learning rate decay       & 0.99                  & 0.99                        & 0.99                & 0.99             \\
learning rate decay steps & 100                   & 100                         & 100                 & 100             \\
image size                & 120x90                & 120x90                      & 120x90              & 120x90                          \\
MLP network size (width x depth) & 1024x4         & 1024x4                      & 2048x4              & 1024x4             \\
Conv. Net.                & 4-layer ConvMaxPool & 4-layer ConvMaxPool & 4-layer ConvMaxPool & 4-layer ConvMaxPool \\
activation                & ReLU                  & ReLU                        & ReLU                & ReLU                          \\
dense layer type          & regular               & regular                     & regular             & regular                          \\
train counter examples    & 256                   & 256                         & 256                 & 256             \\
action boundary buffer    & 0.05                  & 0.05                        & 0.05                & 0.05             \\
gradient penalty          & none                  & none                        & none                & none             \\
dfo samples               & 1024                  & 1024                        & 2048                & 2048             \\
dfo iterations            & 3                     & 3                           & 3                   & 3                          \\
\hline
\end{tabular}
\end{table}

\subsubsection*{Pushing Pixels MSE-BC}

\begin{table}[H]
\scriptsize
\begin{tabular}{|l|l|l|l|l|}
\hline
\textbf{Hyperparameter}   & \textbf{Pushing}      & \textbf{Pushing Multimodal} & \textbf{Insertion} & \textbf{Sorting} \\
\hline
train iterations          & 100,000               & 100,000                     & 100,000            & 100,000             \\
batch size                & 128                   & 128                         & 128                & 128                 \\
sequence length           & 2                     & 2                           & 2                  & 2                 \\
learning rate             & 0.001                 & 0.001                       & 0.001              & 0.001                 \\
learning rate decay       & 0.99                  & 0.99                        & 0.99               & 0.99                 \\
learning rate decay steps & 100                   & 100                         & 100                & 100                 \\
image size                & 120x90                & 120x90                      & 120x90              & 120x90                          \\
dropout rate (MLP only)               & 0.1                   & 0.1                         & 0.1                & 0.1                               \\
MLP network size (width x depth) & 512x4          & 1024x4                      & 1024x4             & 1024x4           \\
Conv. Net.                & 4-layer ConvMaxPool & 4-layer ConvMaxPool & 4-layer ConvMaxPool & 4-layer ConvMaxPool \\
activation                & ReLU                  & ReLU                        & ReLU               & ReLU                              \\
\hline
\end{tabular}
\end{table}

\section{Model Architectures}\label{subsec:architectures}

\begin{figure}[H]
\centering
  \includegraphics[width=1.0\textwidth]{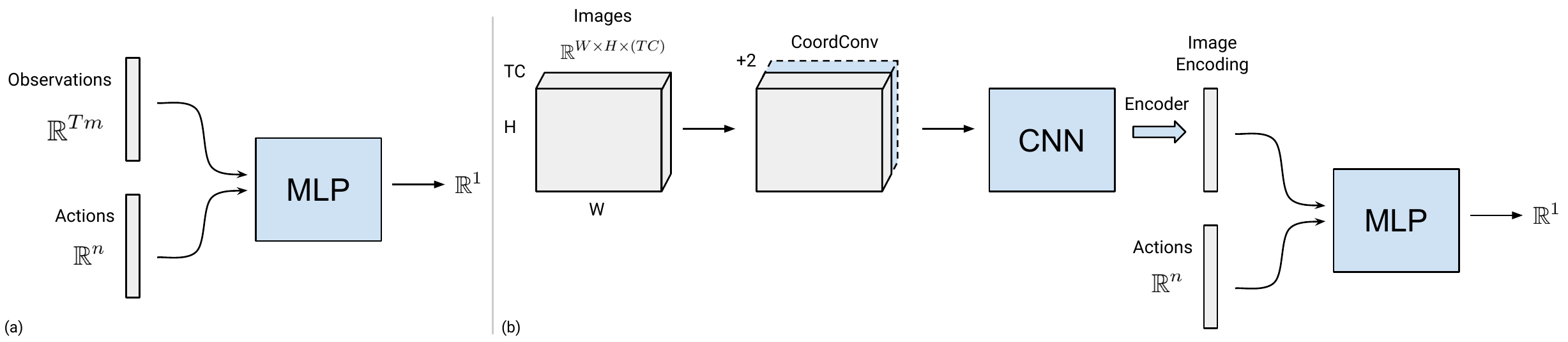}
  \caption{Simple depictions of architectures used for state-observation models (a), and visuomotor models (b).  $T$ is sequence length, $m$ is observation dimensionality, $n$ is action dimensionality, $W$, $H$, $C$ are image width, height and channels. }.
  \label{fig:archs}
  \vspace{-1em}
\end{figure}

\subsection{MLPs}

For non-image-observation models, we use MLPs (Multi Layer Perceptrons) that when used as EBMs (Fig.~\ref{fig:archs}a), take in the actions and output an energy in $\mathbb{R}^1$, or when trained as MSE models instead output the actions.  All results shown used ReLU activations, although we experimented with Swish as well. Configurable model elements consisted of: Dropout \cite{srivastava2014dropout}, using ResNet skip connections \cite{he2016identity}, and spectral normalization dense layers instead of regular dense layers \cite{miyato2018spectral}.

\subsection{ConvMLPs}

For visuomotor models (Fig.~\ref{fig:archs}b), we use the common ConvMLP \cite{levine2016end} style architecture, but when used as an EBM, concatenate actions with image encodings from a CNN model.  The MLP portion is identical to the section above.  For the CNNs, for all models for the sweeping experiments, we used 26-layer ResNets \cite{he2016deep} (``ConvResNets'') which maintain full-image spatial resolution before the encoder. For the simulated and real-world pushing experiments, we used a progressively-spatially-reduced model (``ConvMaxPool'') composed of interleaving convolutions with max-pooling, with feature dimensions [32, 64, 128, 256]. Both models used 3x3 convolution kernels. Configurable options include: using CoordConv \cite{liu2018intriguing}, i.e. a pixel coordinate map augmented as input, and either spatial soft (arg)max \cite{levine2016end} or global average pooling encoders.




\section{Proofs}

In this section we prove Theorems \ref{thm:any-function-theorem} and \ref{universal-approximation-tmp} as stated in Section 5 of the paper.

\begin{figure}[htb]
\begin{subfigure}{.5\textwidth}
  \centering
  \includegraphics[width=.9\linewidth]{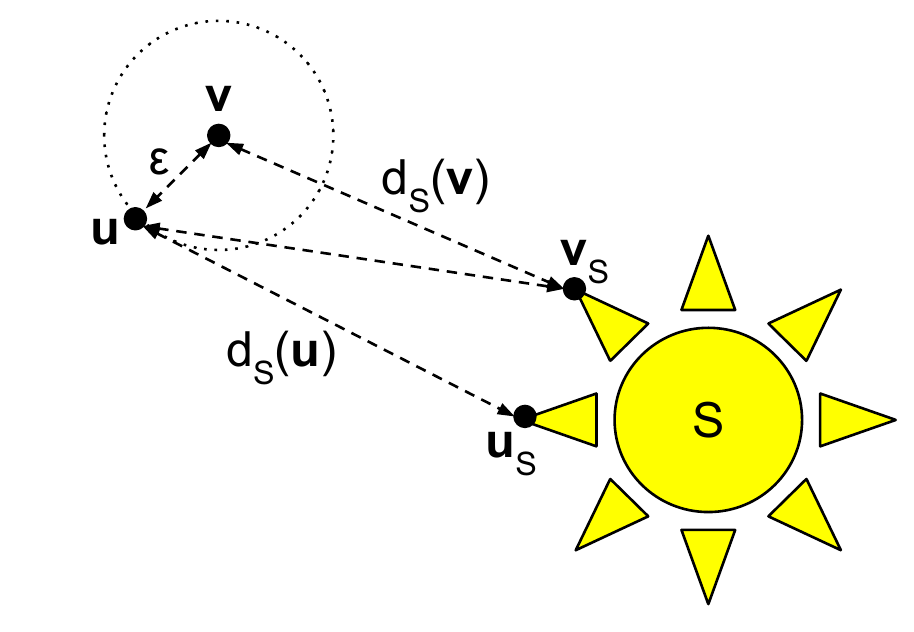}
  \caption{Distance to a set, $S$.}
  \label{fig:dist-to-set-triangle}
\end{subfigure}
\begin{subfigure}{.5\textwidth}
  \centering
  \includegraphics[width=.9\linewidth]{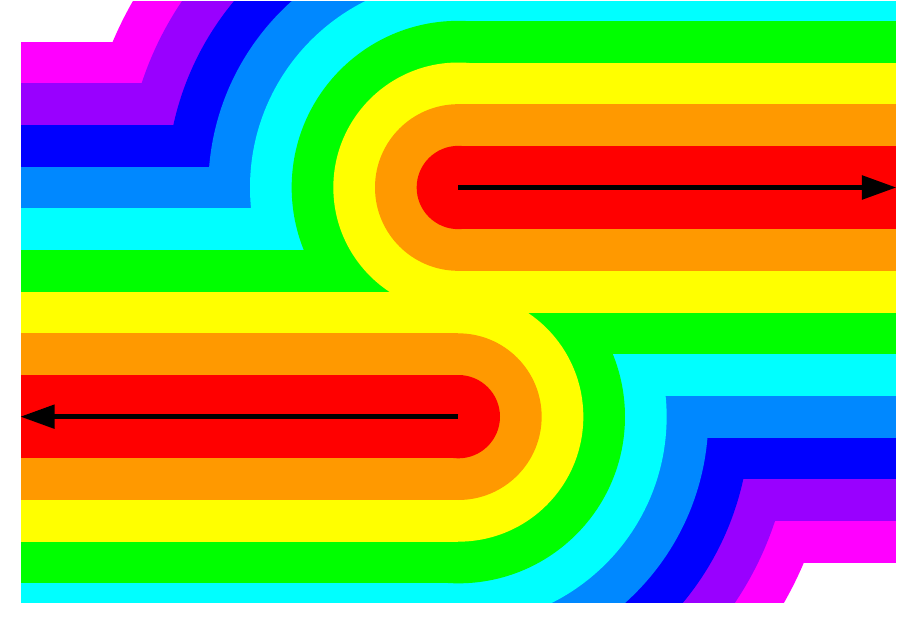}
  \caption{Distance function isarithms for a step function.}
  \label{fig:dist-to-set-func}
\end{subfigure}
\end{figure}

\subsection{Definitions}

A function $f$ is {\em Lipschitz continuous} with constant $L$ if $||f(\bfx) - f(\bfy)|| \leq L ||\bfx - \bfy||$ for all $\bfxy$.
We say that $f$ is $L$-Lipschitz, so a 1-Lipschitz function is a function that is continuous with Lipschitz constant 1.
The magnitude of the gradient of an $L$-Lipschitz function is always less than or equal to $L$.

The {\em distance function} from a point $\bfx \in \Rn$ to a non-empty set, $S \subset \Rn$ is defined as:
\[
d_S(\bfx) = \underset{\bfx'\in S}{\inf}||\bfx-\bfx'||
\]

A {\em closed set} is a set that contains all of its boundary points (points that can be approached from the interior and exterior of the set).
Equivalently, a set if closed if and only if it contains all of its limit points (points that are the limit of some sequence of points in the set).

The {\em power set} of $\Rn$, $P(\Rn)$ is the set of all subsets of $\Rn$ including the empty set and all of $\Rn$.

The {\em{graph}}, $\mathcal{G}_F$, of a function $F:\Rm\to \Rn$ is the set of points:
\[
\mathcal{G}_F\ =\ \{ (\bfxFx)\ \forall\ \bfx\in\Rm\} \subset \Rmn
\]

The {\em{graph}}, $\mathcal{G}_F$, of a multi-valued function $F:\Rm\to P(\Rn)$ is the set of points:
\[
\mathcal{G}_F\ =\ \{ (\bfxy)\in\Rmn\ \mid\ \bfx\in\Rm,\ \bfy\in F(\bfx)\}
\]

\subsection{Proofs}

\begin{lemma}\label{lemma:dist-func-continuous}
The {\em{distance function}} from any point $\bfv$ to a non-empty set, $S \subset \Rn$, is well-defined and 1-Lipschitz.
\end{lemma}
\begin{proof}
The distance function from a point $\bfv$ to a non-empty set, $S$ is defined as:

\[
d_S(\bfv) = \underset{\bfv_S\in S}{\inf}||\bfv-\bfv_S||
\]

The set of distance values is a set of positive real numbers, so the infimum exists due to the {\em completeness} of $\Real$.
Therefore the distance function is well defined.

For any $\bfv$, let $\bfv_S$ be a point in the closure of $S$ with $||\bfv - \bfv_S|| = d_S(\bfv)$.
Then, to establish continuity using the triangle inequality, we can state that
for a given $\bfu$ at a distance $\epsilon$ from $\bfv$ (as pictured in Fig.~\ref{fig:dist-to-set-triangle}),
\begin{align*}
||\bfu - \bfv||\ &=\ \epsilon\\
d_S(\bfu)\ &=\ ||\bfu - \bfu_S|| \\ 
&\leq\ ||\bfu - \bfv_S|| &\text{$d_S$ is an infimum}\\ 
&\leq\ ||\bfu - \bfv|| + ||\bfv - \bfv_S|| &\text{by the triangle inequality}\\ 
d_S(\bfu)\ &\leq\ \epsilon + d_S(\bfv) \\
d_S(\bfv)\ &\leq\ \epsilon + d_S(\bfu) &\text{$\bfv$ and $\bfu$ can be exchanged.}\\
|d_S(\bfu) - d_S(\bfv)|\ &\leq\ \epsilon \\
|d_S(\bfu) - d_S(\bfv)|\ &\leq\ 1 \cdot ||\bfu - \bfv|| \\
\end{align*}

Since $\bfu$ and $\bfv$ can be reversed we have, $|d_S(\bfu) - d_S(\bfv)| < \epsilon$ and thus $d_S$ is continuous over $\Rn$ with a Lipschitz constant of 1.
\end{proof}

\begin{lemma}\label{lemma:dist-func-achieved}
If $d_S:\Rn\to\Real$ is the distance function to a {\em closed} set $S \subset \Rn$, then for every $\bfx\in\Rn$ there exists an element $\bfx'\in S$ such that $d_S(\bfx) = ||\bfx - \bfx'||$.
\end{lemma}
\begin{proof}
Let $B$ be a closed ball of radius $d_S(\bfx) + 1$ around $\bfx$.
The distance from $\bfx$ to $B \cap S$ is equal to the distance from $\bfx$ to $S$.
Since $d_S$ is defined as an infimum, there must exist an infinite sequence of points $\{\bfx_i\} \subset B \cap S$ with distances $d_i = ||\bfx - \bfx_i||$ whose limit is $d_S(\bfx)$.
The set $B \cap S$ is closed and bounded and, therefore, compact.
The infinite sequence $\{x_i\}$ must therefore have at least one sub-sequence that converges to a point $\bfx' \in B \cap S$.
Since the distances of the full series converge to $d_S(\bfx)$, we know that $||\bfx - \bfx'|| = d_S(\bfx)$.
\end{proof}

\begin{lemma}\label{lemma:implicit-continuous}
For any continuous function $F(\bfx):\ \bfx \in \Rm \rightarrow \mathbb{R}^n$, the distance to the graph of $F$
is a continuous function $g(\bfxy): \ \Rmn \rightarrow \Real$,
such that $F_g(\bfx) = \argminy \ g(\bfxy) = F(\bfx)$ for all $\bfx$.
\end{lemma}

\begin{proof}
Let $g(\bfxy)$ be the distance in $\Rmn$ from the point $(\bfxy)$ to the graph of $F$.

The {\em{graph}}, $\mathcal{G}_F$, of a function $F:\Rm\to \Rn$ is the set of points:
\[
\mathcal{G}_F\ =\ \{ (\bfxFx)\ \forall\ \bfx\in\Rm\} \subset \Rmn
\]

Since the graph $\mathcal{G}_F$ is a non-empty set the distance function $g(\bfxy)$ is well defined and continuous, as shown in Lemma \ref{lemma:dist-func-continuous}.

We must still show that $F_g(\bfx) = \argminy \ g(\bfxy) = F(\bfx)$ for all $\bfx$.
We know that $g(\bfxy) \geq 0\ \forall \bfxy$, because $g$ is a distance function.

For any $\bfx \in \Rn$, clearly $g(\bfxFx) = 0$, since the point $(\bfxFx) \in \mathcal{G}_F$ and thus the distance from $(\bfxFx)$ to a point in $\mathcal{G}_F$ is zero.

Consider a point $(\bfxy)$ where $\bfy \neq F(\bfx)$ and therefore $(\bfxy) \notin \mathcal{G}_F$.
Since $F$ is continuous, $\mathcal{G}_F$ is closed and there will exist a point,
$(\bfx', F(\bfx')) \in \mathcal{G}_F$ that achieves the infimum, $d_\mathcal{G}((\bfxy)) = ||(\bfxy) - (\bfx', F(\bfx'))||$.

\[
||(\bfxy) - (\bfx', F(\bfx'))|| = ||(\bfx - \bfx', \bfy - F(\bfx'))|| \leq ||\bfx - \bfx'|| + ||\bfy - F(\bfx')||
\]

At least one of $\bfx \neq \bfx'$ or $\bfy \neq F(\bfx')$, so $g(\bfxFx) > 0$.

Therefore, for any $\bfx \in \Rm$, $g(\bfxy)$ achieves its unique minimum $g(\bfxy) = 0$ at $\bfy = F(\bfx)$ and thus
$F_g(\bfx) = \argminy \ g(\bfxy) = F(\bfx)$.
\end{proof}

\begin{figure}[htb]
\begin{subfigure}{.5\textwidth}
  \centering
  \includegraphics[width=.9\linewidth]{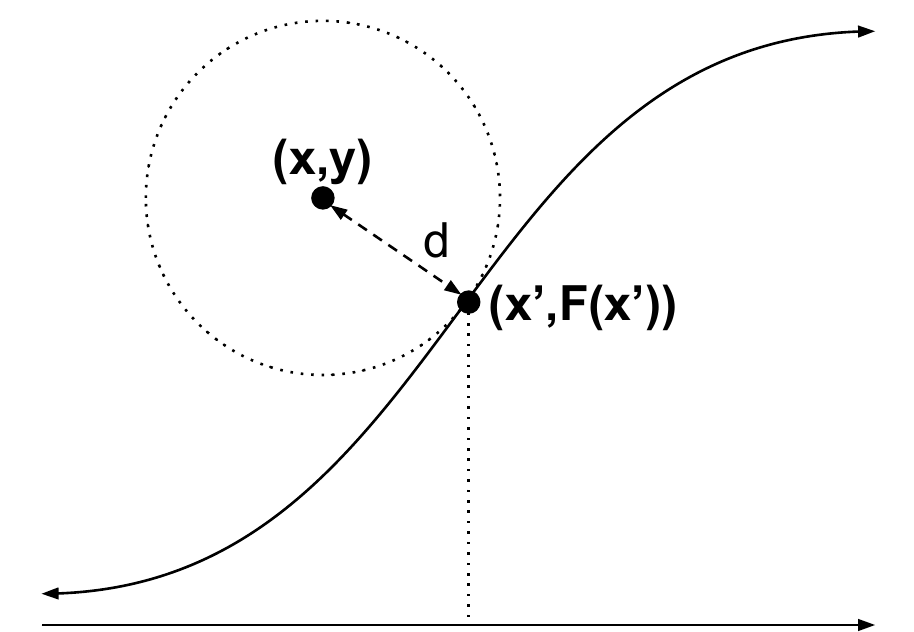}
  \caption{Distance to the graph of a continuous function.}
  \label{fig:dist-to-func}
\end{subfigure}%
\begin{subfigure}{.5\textwidth}
  \centering
  \includegraphics[width=.9\linewidth]{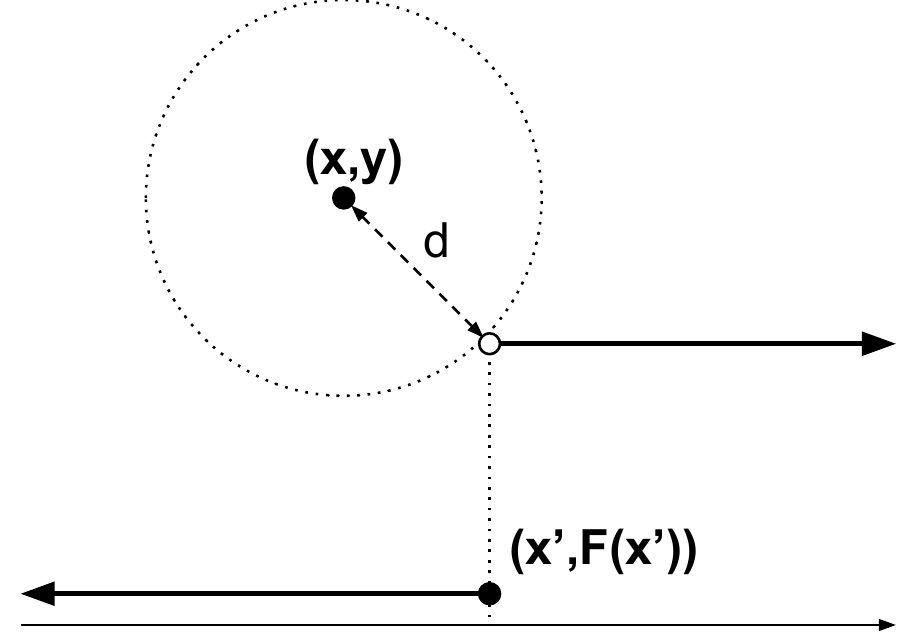}
  \caption{Distance to the graph of a (single-valued) discontinuous function.}
  \label{fig:dist-to-discontinuous-func}
\end{subfigure}
\end{figure}

We have shown that for $F:\Rm\to\Rn$ we can construct a continuous $g(\bfxy)$ that satisfies
$F_g(\bfx) = \argmin_\bfy \ g(\bfxy) = F(\bfx)$ for all $x \in \Rm$ if $F$ is single-valued and continuous.
However, the functions we are modeling are often discontinuous or multi-valued.
If the single-valued function is discontinuous, there will be open boundaries on the graph where the point that minimizes the distance function is not in the graph of $F$ (Fig.~\ref{fig:dist-to-discontinuous-func}).
In that example, there will be two values of $\bfy$ that minimize $g$ for the same value of $\bfx$, in which case $F_g(\bfx)$ will not be well defined as a single-valued function.  We can disambiguate the two cases to get a well-defined $F_g$, but we cannot reliably recover the original $F$ at the discontinuity.

In order to handle discontinuities and multi-valued functions, we will extend the definition to allow functions that map to multiple values, $F: \Rm \to \PRn$.
The multi-valued function $F$ maps from $\mathbb{R}^m$ to the {\em{power set}} $\PRn $, which is the set of all subsets of $\mathbb{R}^n$, except the empty set.
We no longer require continuity, but instead directly require the one important property of a continuous function that was used in the proof of Lemma \ref{lemma:implicit-continuous}, namely that the graph of $F$ is closed.
In the simple case of a jump discontinuity (as in fig. \ref{fig:dist-to-discontinuous-func}), the function must include both sides of the discontinuity.

\begin{customTheorem}{1}\label{thm:any-function-theorem}
For any multi-valued (set-valued) function $F(\bfx):\ \bfx \in \Rm \rightarrow \PRn$ where the graph of $F$ is closed,
there exists a 1-Lipschitz function $g(\bfxy): \ \Rmn \rightarrow \Real$,
such that $\argminy \ g(\bfxy) = F(\bfx)$ for all $\bfx$.
\end{customTheorem}

\begin{proof}
The {\em{graph}}, $\mathcal{G}_F$, of a multi-valued function $F:\Rm\to P(\Rn)$ is the set of points:
\[
\mathcal{G}_F\ =\ \{ (\bfxy)\in\Rmn\ \mid\ \bfx\in\Rm,\ \bfy\in F(\bfx)\}
\]

We can again define $g$ as the distance to $\mathcal{G}_F$.
Because $\mathcal{G}_F$ is a non-empty set, we know that $g$ is well-defined and uniformly continuous (Lemma \ref{lemma:dist-func-continuous}).
\[
g(\bfxy)\ =\ d_{\mathcal{G}_F}((\bfxy))
\]

We will now show that $g(\bfxy) = 0$ for all points in $\mathcal{G}_F$ and $g(\bfxy) > 0$ for all points not in $\mathcal{G}_F$.

For any point $(\bfxy)\in\mathcal{G}_F$, that is $\bfx\in\Rm$ and $\bfy\in F(\bfx)$, the distance from $(\bfxy)$ to $\mathcal{G}_F$ is zero.
\[
g(\bfxy)\ =\ d_{\mathcal{G}_F}((\bfxy))\ =\ 0
\]

For any point $(\bfxy)\notin\mathcal{G}_F$, that is $\bfx\in\Rm$ and $\bfy\notin F(\bfx)$, we must show that the distance to $\mathcal{G}_F$ is strictly positive.
Since the graph $\mathcal{G}_F$ is closed (because we require it to be so), by Lemma \ref{lemma:dist-func-achieved}, we know that
there exists a point $(\bfx',\bfy')\in\mathcal{G}_F$ that achieves the minimum distance exactly.

\begin{align*}
g(\bfxy) \ &=\ d_{\mathcal{G}_F}((\bfxy)) \\
&=\ ||(\bfxy) - (\bfx', \bfy')||\\
&=\ ||(\bfx - \bfx', \bfy - \bfy')|| \\
&\leq\ ||\bfx - \bfx'|| + ||\bfy - \bfy'||
\end{align*}

At least one of $\bfx \neq \bfx'$ or $\bfy \neq \bfy'$, so $g(\bfxy) > 0$.

Because the empty set is excluded from the range of $F$, there will be at least one $\bfy\in F(\bfx)$ for any $\bfx$.
Therefore, at $\bfx$, the minimum value of $g(\bfxy)$ will be zero and $\argminy g(\bfxy) = F(\bfx)$ exactly.

We have shown that the implicit function $g(\bfxy)$ is well-defined and continuous and has a Lipschitz value of 1, even if $F$ is very badly behaved.
\end{proof}

\textbf{Corollary 1.1}{\em{
The function $g(\mathbf{x}, \mathbf{y})$ in Thm.~\ref{thm:any-function-theorem} can be chosen to have an arbitrary positive Lipschitz constant.
}}

%
\begin{proof}
The distance function $d_S(\bfx):\Rn\to\Real$ from any point $\bfx$ to a non-empty set $S$,  has a Lipschitz constant of 1 (Lemma \ref{lemma:dist-func-continuous}).
Let $g_1\ =\ d_{\mathcal{G}_F}$, the distance to the graph of $F$.
If our desired Lipschitz constant is $L > 0$, we can compose $g_1$ with another function $f_L:\Real\to\Real$ that has a Lipschitz constant of $L$ to get $g_L = f_L \circ g_1$.
For example, if $f_L(x) = L x$, we get $g_L(\bfxy) = f_L(g_1(\bfxy)) = L\ d_{\mathcal{G}_F}(\bfxy)$.

\begin{align*}
    |g_L(\bfxy) - g_L(\bfx',\bfy')|
\ &=\ |f_L(g_1(\bfxy)) - f_L(g_1(\bfx',\bfy'))| \\ 
\ &\leq\ L\ |g_1(\bfxy) - g_1(\bfx',\bfy')| \\ 
\ &\leq\ L\ \cdot\ 1\ ||(\bfxy) - (\bfx',\bfy')|| 
\end{align*}
Therefore, $g_L$ has a Lipschitz constant of $L$.
\end{proof}


\begin{customTheorem}{2}\label{universal-approximation-tmp}
For any set-valued function $F(\bfx):\Rm \rightarrow \PRn$, there exists a continuous implicit function $g:\Rmn\to\Real$ that has a {\em continuous function approximator}, $g_{\theta}$ with arbitrarily small bounded error $\epsilon$,
that provides the guarantee that any point in the graph of $F_\theta(\bfx) = \argminy g_\theta(\bfxy)$ is within $\epsilon$ of the graph of $F$.
\end{customTheorem}

\begin{proof}
Let $g(\bfxy) = 2 d_{\mathcal{G}_F}(\bfxy)$, twice the distance to the graph of $F$.  By Thm. 1, this $g$ is continuous and satisfies $\argminy g(\bfxy)=F(\bfx)$.

For an arbitrary $\epsilon > 0$, let $g_\theta:\Rmn\to\Real$ be a function approximator for $g$ with bounded error $\epsilon$, $|g_\theta - g| < \epsilon$.
Since $g$ is a continuous function, the existence of a bounded-error function approximator 
is guaranteed by well-known results in universal approximation of continuous functions, for example \cite{cybenko1989approximation}.

The question now is whether bounded errors in $g_{\theta}$ approximating $g$, when composed with $\argmin$, can provide any guarantee on a property of $\argminy g_{\theta}(\bfxy)$.

Note that $F_\theta$ as an approximator for $F$ is unbounded, since $F$ may be badly behaved.
This can be demonstrated at any point where $F$ has a discontinuity.
Suppose $||\bfx-\bfx'||<\epsilon$ and $F(\bfx)$ and $F(\bfx')$ have values that are arbitrarily far apart.
Because of the $\epsilon$ error in $g_\theta$, the $\argminy g_\theta(\bfxy)$ may find values in $F(\bfx')$ that introduce arbitrary error in $F_\theta(\bfx)$.

Let $F_\theta(\bfx) = \argminy g_\theta(\bfxy)$.
For any point $(\bfxy)$ in the graph of $F_\theta$, we can show that $g_\theta(\bfxy) < \epsilon$.
Let $\bfy'$ be any point in $F(\bfx)$.
Since $(\bfxy')$ is in the graph of $F$, we know that $g(\bfxy') = 0$. 
With the bounded error, $|g_\theta - g| < \epsilon$, we know that $g_\theta(\bfxy') < \epsilon$.
Thus, although the argmin may be achieved elsewhere, $g_\theta(\bfxy)$ may not have a value greater than $\epsilon$.

\begin{align*}
g_\theta(\bfxy)\ & \leq\ g_\theta(\bfxy') \\
& <\ g(\bfxy') + \epsilon \\
g_\theta(\bfxy)\ &<\ \epsilon
\end{align*}

Also because $g_\theta$ is an approximator for $g$, we know that for ($\bfxy)$ in the graph of $F_{\theta}$, $g(\bfxy) < g_\theta(\bfxy) + \epsilon$ and thus $g(\bfxy) <\ 2 \epsilon$.
Since $g(\bfxy)$ is twice the distance to the graph of $F$, the distance from $(\bfxy)$ to the graph of $F$ is less than $\epsilon$.
\begin{align*}
g(\bfxy)\ &\leq\ g_\theta(\bfxy) + \epsilon\ \leq\ 2 \epsilon \\
g(\bfxy)\ &=\ 2 d_{\mathcal{G}_F}(\bfxy)\ \leq\ 2 \epsilon \\
d_{\mathcal{G}_F}(\bfxy)\ &\leq\ \epsilon \\
\end{align*}
Therefore, any point $(\bfxy)$ in the graph of $F_\theta$ must lie within $\epsilon$ of $\mathcal{G}_F$, the graph of $F$.

Note that this is not symmetric.
There are no guarantees on $g_\theta$ other than that it is continuous and within $\epsilon$ of $g$.
A point in $\mathcal{G}_F$ may be arbitrarily far from $\mathcal{G}_{F_\theta}$ as tiny variations in $g_\theta$ can eliminate some values from the set $F(\bfx)$.
\end{proof}




\section{Theory Implications and Discussion}

The practical implications of Theorems \ref{thm:any-function-theorem} and \ref{universal-approximation-tmp} are that they provide a number of favorable properties for real-world modeling tasks, such as robot policy learning, that exhibit discontinuities and multi-modalities. For one, implicit functions can model steep or discontinuous oracle policies without large gradients in the function approximator. We hypothesize this leads to policies with better stability characteristics and fewer generalization issues for out-of-domain samples. 
Thm.~1 also shows that implicit models can represent represent multi-valued (set-valued) functions, including ones with discontinuities.  
With Thm.~2 there is a notion provided of still having guarantees on the output of implicit inference, despite expected errors in the function approximator.  One way to consider the guarantee is via the level sets of the function, as shown in Figure 10 of the main paper.  To discuss an illustrative example, for a simple step function discontinuity (as in Figure 10) this guarantee provides that, although the precise ``decision boundary'' of the discontinuity may not be represented perfectly, that decision boundary can be approximated to arbitrary precision.  Further, even when the decision boundary is estimated imperfectly, the inferred values will correspond to either side of the discontinuity, and nowhere else (it will not, for example estimate a value somewhere halfway between the two sides of the discontinuity). 

Note also that Theorems 1 and 2 do not show that a learning algorithm will actually be able to recover a model with such properties, but only that such a model exists.  This is in line with other results in function approximation with neural networks, for example \cite{cybenko1989approximation}.  Additionally, in the implicit model case another consideration is not only whether the parameters $\theta$ of such a parameterized model, $g_{\theta}$ can be found, but then also whether at inference time the optimization problem $\argmin_\bfy g_{\theta}(\bfx, \bfy)$ can even be solved.  In general, $g(\cdot)$ will be non-convex, and may be a hard global optimization problem.  We have shown in practice, however, we are able to perform satisfactory inference on many policy learning tasks of interest.

\section{Limitations}

Although we have examined performance increases on several tasks using implicit BC policies over explicit BC policies, evaluated as a method for behavioral cloning, there are a few limitations. For one, compared to a simple Mean Square Error (MSE) BC policy, there is both increased training and inference computational cost.  As shown in Sec.~\ref{sec:timing}, however, the increase in inference time for the models used in the real world is modest, and we have validated the models in the real world to be fast enough for real-time vision-based control.  Further, the training times of our presented models are modest compared to reported numbers for offline RL methods (Sec.~\ref{sec:timing}) A second limitation is the increased implementation complexity of implicit models compared to explicit models.  However, we have provided a guide (Sec.~\ref{sec:ebm_variants_table}) for how to train the models, which we hope encourages readers to try implicit models.










\end{document}